%% file: main.tex
\documentclass{article}

\usepackage{microtype}
\usepackage{graphicx}
\usepackage{subfigure}
\usepackage{booktabs} 

\usepackage{hyperref}
\usepackage{multirow}
\usepackage{enumitem}
\usepackage{listings}
\usepackage{algorithm}
\usepackage{algorithmic}
\usepackage{caption}
\usepackage[normalem]{ulem}
\useunder{\uline}{\ul}{}
\usepackage{natbib}
\setcitestyle{authoryear,round,citesep={;},aysep={,},yysep={;}}

\usepackage[preprint]{neurips_2024}

\usepackage{amsmath}
\usepackage{amssymb}
\usepackage{mathtools}
\usepackage{amsthm}

\usepackage[capitalize,noabbrev]{cleveref}

\theoremstyle{plain}
\newtheorem{theorem}{Theorem}[section]
\newtheorem{proposition}[theorem]{Proposition}

\newtheorem{corollary}[theorem]{Corollary}
\theoremstyle{definition}
\newtheorem{definition}[theorem]{Definition}

\theoremstyle{remark}
\newtheorem{remark}[theorem]{Remark}

\usepackage[textsize=tiny]{todonotes}

\DeclareMathOperator{\softmax}{softmax}
\DeclareMathOperator{\grad}{grad}
\DeclareMathOperator{\logsumexp}{logsumexp}
\newcommand*\diff{\mathop{}\!\mathrm{d}}
\DeclareMathOperator*{\esssup}{ess\,sup}

\title{$\alpha$-Flow: A Unified Framework for Continuous-State Discrete Flow Matching Models}

\author{
Chaoran Cheng$^1$,\,\,
Jiahan Li$^2$,\,\,
Jiajun Fan$^1$,\,\,
Ge Liu$^1$,\,\,\\
$^1$University of Illinois Urbana-Champaign,\,\,
$^2$Tsinghua University\\
Corresponding email: \texttt{chaoran7@illinois.edu}
}

\begin{document}

\maketitle

\input{secs/0_abstract}

\input{secs/1_introduction}

\input{secs/2_prelim}

\input{secs/3_method}

\input{secs/4_experiment}

\input{secs/5_related}

\input{secs/6_conclusion}


\input{secs/7_impact}

\bibliography{bibs/flow,bibs/infogeo,bibs/others}
\bibliographystyle{plainnat}

\newpage
\appendix
\onecolumn
\centerline{\Large\bf Supplementary Material}

\input{suppl/A_riemann}

\input{suppl/B_infogeo}

\input{suppl/C_proof}

\input{suppl/D_model}

\input{suppl/E_experiment}

\input{suppl/F_results}


\end{document}

%% file: secs/0_abstract.tex
\begin{abstract}
Recent efforts have extended the flow-matching framework to discrete generative modeling. One strand of models directly works with the continuous probabilities instead of discrete tokens, which we colloquially refer to as Continuous-State Discrete Flow Matching (CS-DFM). Existing CS-DFM models differ significantly in their representations and geometric assumptions. 
This work presents a unified framework for CS-DFM models, under which the existing variants can be understood as operating on different $\alpha$-representations of probabilities. Building upon the theory of information geometry, we introduce $\boldsymbol{\alpha}$-\textbf{Flow}, a family of CS-DFM models that adheres to the canonical $\alpha$-geometry of the statistical manifold, and demonstrate its optimality in minimizing the generalized kinetic energy.
Theoretically, we show that the flow matching loss for $\alpha$-flow establishes a unified variational bound for the discrete negative log-likelihood. 
We comprehensively evaluate five different instantiations of $\alpha$-flow on various discrete generation domains to demonstrate their effectiveness in discrete generative modeling, including intermediate values whose geometries have never been explored before.
$\alpha$-flow significantly outperforms its discrete-state counterpart in image and protein sequence generation and better captures the entropy in language modeling.
\end{abstract}

%% file: secs/1_introduction.tex
\section{Introduction}
The success of flow matching \citep{lipman2022flow} and diffusion \citep{ho2020denoising} models has inspired recent attempts to adopt such generative frameworks to discrete domains.
Among the existing discrete flow matching (DFM) models, one strand of models directly works with the continuous representations of probabilities instead of discrete tokens, which we colloquially refer to as \emph{Continuous-State} DFM (CS-DFM). To leverage the continuous flow matching framework for the generative modeling of discrete data, all CS-DFM models face two key designs of:
\begin{enumerate}[topsep=0pt,itemsep=-1ex,partopsep=1ex,parsep=1ex]
    \item The continuous \emph{representation} of the discrete data; and
    \item The \emph{geometry} of the continuous representation.
\end{enumerate}
The continuous representation allows direct applications of flow matching, and the geometry defines the conditional flow on which the vector field can be regressed.

Existing CS-DFM models differ significantly in their choices of these two modeling approaches.
Similar to autoregressive \citep{radford2019language} or masked language models \citep{devlin2018bert}, the most straightforward choice of the continuous representations might be the logits of the categorical probabilities. Alternatively, one may work directly with probabilities with discrete data represented as a one-hot vector. Existing works often assume a Euclidean geometry such that the intermediate representation can be obtained by linear interpolation between the initial noise and the target representations \citep{li2024full,song2023equivariant}.
Two recent works \citep{cheng2024categorical,davis2024fisher} incorporated mathematical results from information geometry and designed a continuous representation of the categorical distribution on the unit sphere with the canonical spherical geometry.
In principle, different representations and geometries lead to different variants of CS-DFM. With such flexibility and diversity, no framework has yet been proposed to provide a unified and comparative analysis of the existing variants mentioned above.

This work fills in the missing gap for connecting existing CS-DFM models and presents a unified framework under which variants can be understood as operating on the different $\alpha$-\emph{representations} of the categorical distribution. Furthermore, rigorously built upon the theory of information geometry, we propose a family of CS-DFM models that adheres to the canonical $\alpha$-\emph{geometry} by following the $\alpha$-\emph{geodesics}, which we call $\boldsymbol\alpha$-\textbf{Flow}. Our $\alpha$-flow enjoys additional theoretical benefits as we demonstrate its optimality with respect to the generalized kinetic energy, encompassing some already-known results in previous work as special cases. Additionally, we establish a unified variational bound for the discrete negative log-likelihood (NLL), providing theoretical justifications for the flow matching loss in discrete generative modeling. 
We extensively evaluate different instantiations of $\alpha$-flow on the discrete generative modeling tasks in various domains, including computer vision, natural language processing, and bioinformatics, scaling up to billion-level tokens to study their empirical performance. Our $\alpha$-flow outperforms its discrete-state counterpart and is on par with the latter in language modeling. Heuristic rules for choosing proper model designs are also discussed.


%% file: secs/2_prelim.tex
\section{Preliminary}

\subsection{Riemannian Flow Matching}
Conditional flow matching (CFM) \citep{lipman2022flow} addresses generative modeling by learning a time-dependent vector field that pushes the prior noise distribution to any target data distribution. Such a flow-based model can be viewed as the continuous generalization of the diffusion model \citep{song2019generative,song2020improved,ho2020denoising} with a more flexible design of the denoising process. 
The Riemannian flow matching framework \citep{chen2023riemannian} further extends CFM to general manifolds on which a distance metric can be efficiently computed.
Mathematically, on a smooth Riemannian manifold $\mathcal{M}$ with the Riemannian metric $g$, a \emph{flow} $\psi_t:[0,1]\times \mathcal{M}\to\mathcal{M}$ is a time-dependent diffeomorphism defined by a time-dependent vector field $u_t:[0,1]\times \mathcal{M}\to T\mathcal{M}$ via the ordinary differential equation (ODE): $\frac{\diff}{\diff t}\psi_t(x)=u_t(\psi_t(x))$. The flow matching objective directly regresses the vector field $u_t$ with a time-dependent neural net. Both \citet{lipman2022flow,chen2023riemannian} demonstrated that a tractable objective can be derived by conditioning on the target data $x_1$ at $t=1$ of the probability path. In this way, the Riemannian flow matching objective can be formulated as
\begin{equation}
    \mathcal{L}_\text{RFM}=\mathbb{E}_{t,p_0(x),q(x)}\|v(x_t,t)-u_t(x_t|x_0,x_1)\|_g^2 \label{eqn:rfm_loss}
\end{equation}
where $q$ is the data distribution, $p_0$ is the prior distribution, and $x_t:=\psi_t(x|x_0,x_1)$ denotes the conditional flow. \citet{chen2023riemannian} further demonstrated that if the exponential map and logarithm map can be evaluated in closed-form, the condition flow can be defined as $x_t=\exp_{x_0}(t\log_{x_0}x_1)$, and the corresponding vector field can be calculated as $u_t(x_t|x_0,x_1)=\log_{x_t}(x_1)/(1-t)$.
We will adopt the notations in Riemannian geometry in this work. We use $\gamma_{x_0,x_1}(t):=\psi_t(x|x_0,x_1)$ to denote the geodesic from $x_0$ to $x_1$ and $\dot{\gamma}_{x_0,x_1}(t):=u_t(x_t|x_0,x_1)$ to denote the vector field along the geodesic.

\subsection{Statistical Manifold}
It is known from information theory \citep{rao1992information,amari2000methods,ay2017information} that any parameterized family of probability distributions $p_\theta$ forms a Riemannian structure known as the \emph{statistical manifold} whose canonical Riemannian metric is the \textit{Fisher information metric}.
As we are interested in discrete generative modeling, we naturally focus on the statistical manifold of categorical distributions. Mathematically, consider the discrete sample space $\mathcal{X}=\{1,2,\dots,n\}$, a categorical distribution $\mu\in\mathcal{P}(\mathcal{X})$ can be defined by its categorical probabilities $\mu:\sum_{i=1}^n \mu_i=1,0\le \mu_i\le 1$.
A vector field $a$ in the tangent space $T_\mu\mathcal{P}$ can be identified with $\sum_{i=1}^n a_i=0$.
The canonical inner product at a positive categorical distribution $\mu\in\mathcal{P}_+$, in this case, can be written as
\begin{equation}
    \langle a,b\rangle_\mu=\sum_{i=1}^n \frac{a_i b_i}{\mu_i},\quad a,b\in T_\mu \mathcal{P} .\label{eqn:inner}
\end{equation}
We also use $\|a\|_g=\sqrt{\langle a,a\rangle_\mu}$ to denote the canonical Riemannian norm.
The geometric view of statistical manifolds allows us to derive crucial geometric concepts of exponential and logarithm maps for applying Riemannian flow matching. Specifically, we may define \emph{geodesics}, curves on the Riemannian manifold where a point travels at a ``constant speed''. We note that, even if equipped with the same Fisher information metric, it is also possible to define different families of geodesics on the statistical manifold using different \emph{affine connections} (see Appendix \ref{supp:riemannian}). Previous works \citep{cheng2024categorical,davis2024fisher} already explored the Levi-Civita geodesics for which the traditional definition of local length-minimizing curves is recovered. Our work extends this idea to general families of $\alpha$-geodesics for flow-based generative modeling.

%% file: secs/3_method.tex
\section{Method}
\subsection{$\alpha$-Representation of Categorical Distribution}
We first introduce $\alpha$-representation, a family of continuous representations of the categorical distribution.
\begin{definition}[$\alpha$-representation]\label{def:alpha_rep}
For $\alpha\in[-1,1]$, the $\alpha$-\emph{representation} of a positive categorical distribution $\mu\in\mathcal{P}_+$ is defined by the embedding $\pi^{(\alpha)}:\mathcal{P}_+\hookrightarrow\mathbb{R}_+^n$:
\begin{equation}
\mu\mapsto x=
\begin{cases}
    \mu^\frac{1-\alpha}{2}, &\alpha\ne 1\\
    \log \mu, &\alpha=1
\end{cases} .
\label{eqn:mapping}
\end{equation}
\end{definition}
This perspective allows us to categorize existing CS-DFM methods by their choice of representations: raw probability ($\alpha=-1$), logit ($\alpha=1$), or half-density $\sqrt{\mu}$ ($\alpha=0$, \citet{cheng2024categorical,davis2024fisher}).
A detailed classification is deferred to Section~\ref{sec:classification} and Table~\ref{tab:class}. 
Existing methods often assume a linear interpolation path between the $\alpha$-representations \citep{song2023equivariant,li2024full}. Although such a Euclidean assumption is always available, we demonstrate that a more canonical and mathematically rigorous choice of geometry can be introduced.

\subsection{$\alpha$-Geometry on Statistical Manifold} \label{sec:alpha_geometry}
$\alpha$-representations are closely related to $\alpha$-\emph{connections} that define different geometry on the statistical manifold (see Appendix~\ref{supp:alpha_conn}). Let $p=2/(1-\alpha)$, the $\alpha$-representation lies on the positive orthant of the unit $L_p$ sphere. Mathematical literature \citep{morozova1991markov,ay2017information} demonstrated that the $\alpha$-geodesics adopt the following general form:
\begin{equation}
     \gamma^{(\alpha)}_{\mu,\nu}(t)^{1/p}=\frac{(1-\tau_{\mu,\nu}(t))\mu^{1/p}+\tau_{\mu,\nu}(t)\nu^{1/p}}{\|(1-\tau_{\mu,\nu}(t))\mu^{1/p}+\tau_{\mu,\nu}(t)\nu^{1/p}\|_p}
\end{equation}
where $\tau_{\mu,\nu}$ is an appropriate reparameterization of time with $\tau_{\mu,\nu}(0)=0,\tau_{\mu,\nu}(1)=1$. The exponential map and logarithm map can be calculated as:
\begin{align}
    \exp^{(\alpha)}_\mu(a)^{1/p}&=\frac{\mu^{1/p}+\frac{\tau_{\mu,a}(t)\mu^{1/p}}{p\mu}a}{\|\mu^{1/p}+\frac{\tau_{\mu,a}(t)\mu^{1/p}}{p\mu}a\|_p},\label{eqn:ori_exp}\\
    \log^{(\alpha)}_\mu(\nu)&=\dot\tau_{\mu,\nu}(0) p P_\mu \left(\mu(\nu/\mu)^{1/p}\right),\label{eqn:ori_log}
\end{align}
where $P_\mu(w)=w-\mu\sum_j w_j$ denotes the orthogonal projection onto the tangent space $T_\mu\mathcal{P}$ and $\tau_{\mu,a}$ is the compatible reparameterization with $\tau_{\mu,a}(0)=0,\dot\tau_{\mu,a}(1)=1$. The limit case of $\alpha=1,p=\infty$ is discussed in Appendix~\ref{supp:alpha_limit}.
We call such a flow matching model on the $\alpha$-geometry $\alpha$-\emph{flow}.
In this way, the $\alpha$-flow loss can be written as
\begin{equation}
    \mathcal{L}^{(\alpha)}=\mathbb{E}_{t,p_0(\mu),q(\mu)}\|v_\theta(\gamma^{(\alpha)}(t),t)-\dot{\gamma}^{(\alpha)}(t)\|_g^2, \label{eqn:loss}
\end{equation}
where $\gamma^{(\alpha)}$ denotes the $\alpha$-geodesic from $\mu_0$ to $\mu_1$ and $\dot\gamma^{(\alpha)}$ denotes the vector field along the $\alpha$-geodesic. 
The different exponential maps and interpolations are visualized in Figure \ref{fig:geodesic}. One may observe the different curvatures of the underlying geometries when $\alpha$ varies.

\begin{figure}
    \centering
    \includegraphics[width=0.49\linewidth]{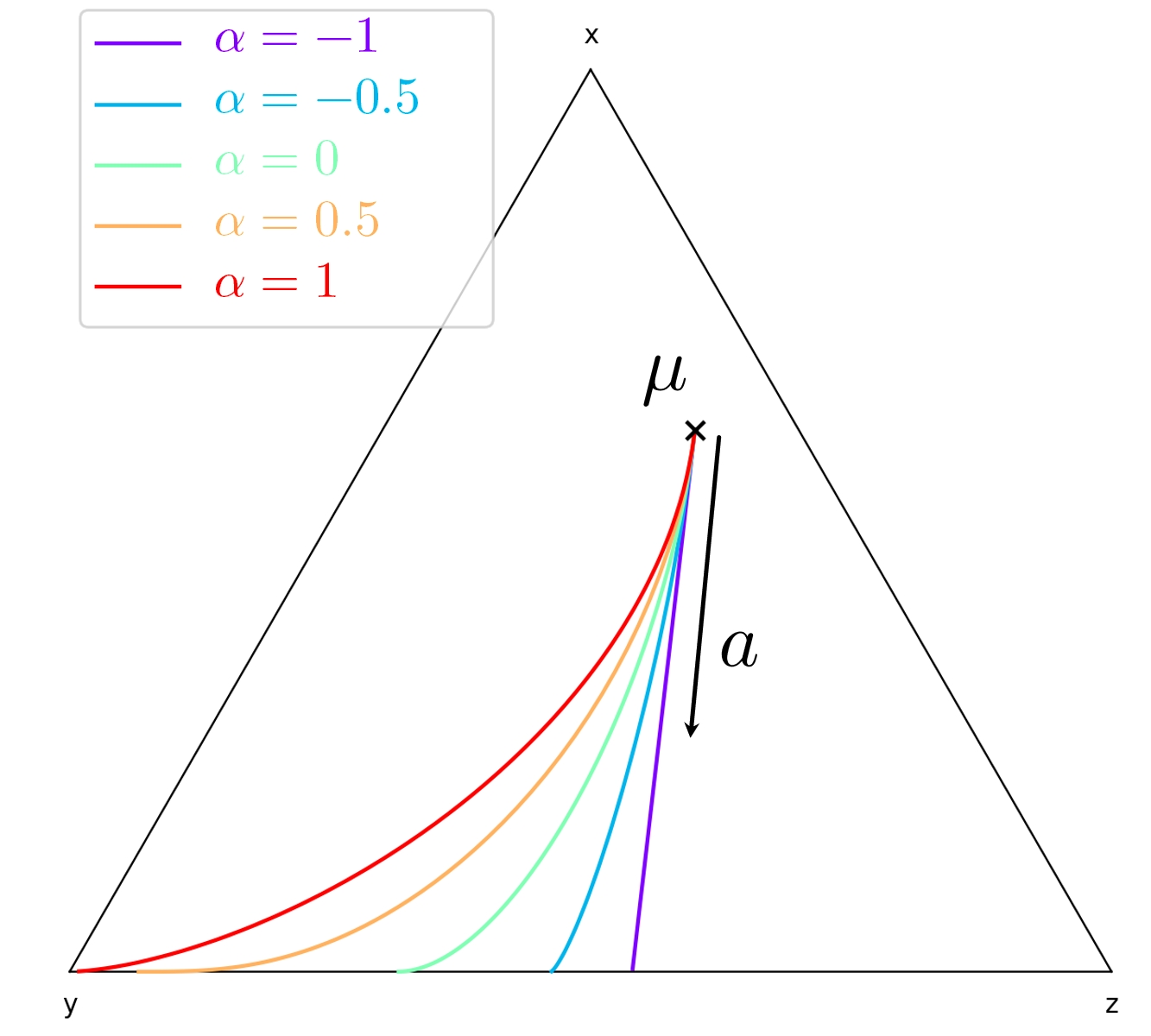}
    \includegraphics[width=0.49\linewidth]{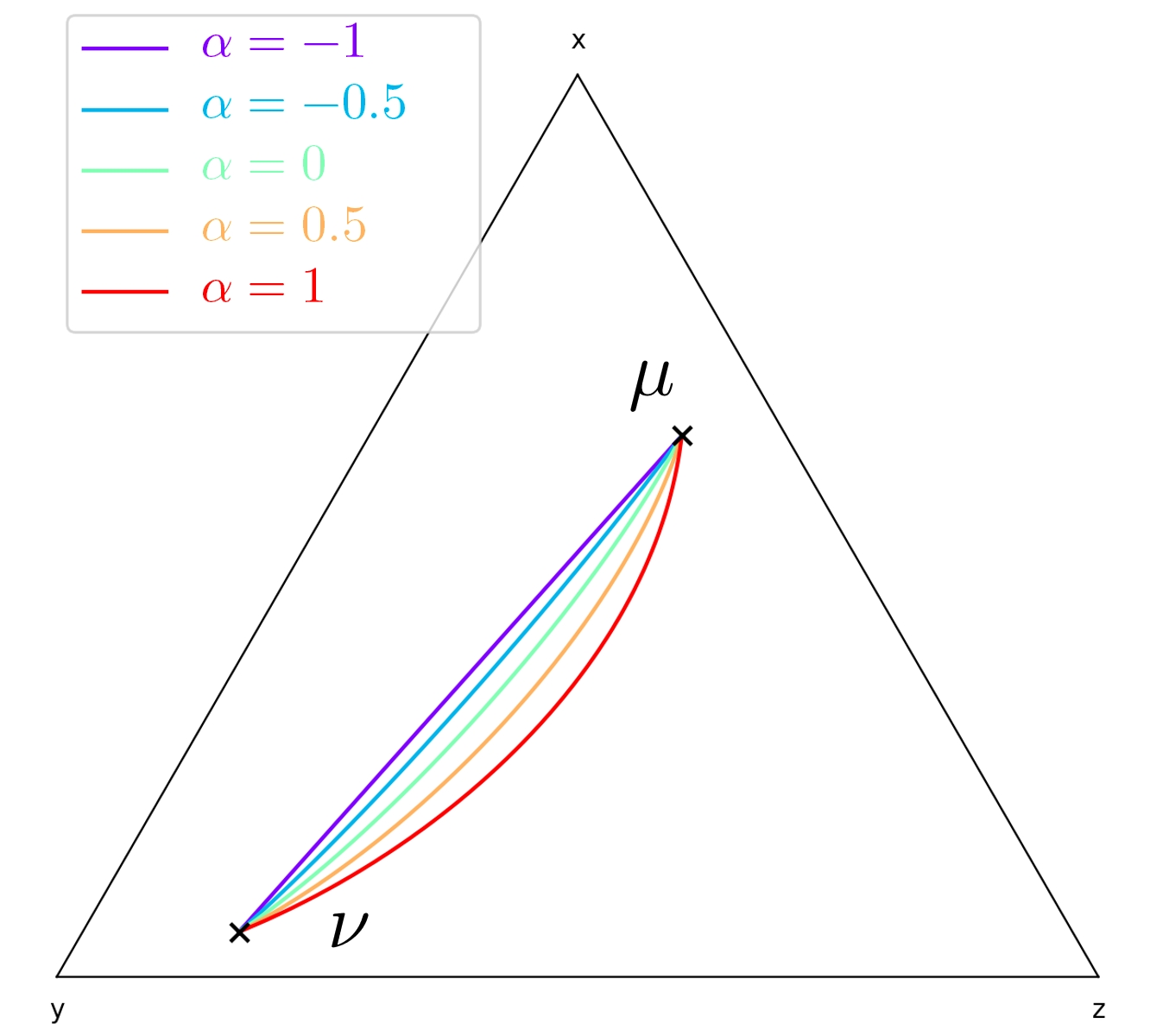}
    \caption{$\alpha$-geodesics defined by the exponential map (\textbf{left}) of the same base point and vector field, and the interpolation (\textbf{right}) between two fixed points on the 2-simplex.}
    \label{fig:geodesic}
    \vspace{-0.5em}
\end{figure}

In principle, a neural net can be directly trained to regress the vector field above. Nonetheless, the division by $\mu$ in the Riemannian norm may lead to numerical issues. Similar to previous work \citep{cheng2024categorical,davis2024fisher}, we can also parameterize the vector field on the mapped manifold of the $L_p$ sphere to avoid such issues. Consider the vector field mapping by taking the derivative with respect to $t$ on both sides of Eq.\ref{eqn:mapping}:
\begin{equation}
u_t:=\frac{\diff}{\diff t}x_t=
\begin{cases}
    \frac{\mu_t^{1/p}}{p\mu_t} a_t, &\alpha\ne 1\\
    \frac{1}{\mu_t}a_t, &\alpha=1    
\end{cases}
. \label{eqn:vf_mapping}
\end{equation}
The transformed vector field $u$ lies in the tangent space $T_x S_p$ of the $L_p$ sphere at $x={\pi}^{(\alpha)}(\mu)$. In this way, the exponential and logarithm maps have simplified forms as
\begin{align}
    &\widetilde{\exp}^{(\alpha)}_x(u)=\frac{x+\tau_{\mu,a}(t)u}{\|x+\tau_{\mu,a}(t)u\|_p},\label{eqn:mapped_exp} \\ 
    &\widetilde{\log}^{(\alpha)}_x(y)=\dot\tau_{\mu,\nu}(0)P^{(\alpha)}_\mu(y-x), \label{eqn:mapped_log}
\end{align}
where 
\begin{equation}
    P^{(\alpha)}_\mu(w)=w-\mu^{1/p}\sum_j\mu^{1-1/p}_j w_j=w-x\sum_j x^{p-1}_jw_j \label{eqn:proj_vf}
\end{equation}
denotes the orthogonal projection to the tangent space $T_x S_p$.
We note the $\alpha$-geometry of the $\alpha$-representations are \emph{almost-linear} in the exponential and logarithm maps in Eq.\ref{eqn:mapped_exp} and \ref{eqn:mapped_log}, only with an additional reparameterization of time and a projection to ensure normalization. 
This is, in fact, a direct result of normalizing the linear $\alpha$-geodesic for positive measures (see Appendix~\ref{supp:alpha_rep_m}).
The original Riemannian norm $\|a\|_g^2=\sum_{i}a_i^2/\mu_i$ can be substituted with 
\begin{equation}
    \|u\|_\alpha^2:=\|a\|_g^2=p^2\sum_{i=1}^n u_i^2\mu_i^{\alpha} . \label{eqn:mapped_norm}
\end{equation}
With this parameterization, similar to the original flow matching, $\alpha$-flow is expected to learn the marginalized linear difference of the $\alpha$-representations of $y-x$.
We summarize the training of $\alpha$-flow in Algorithm~\ref{alg:train}.

\begin{algorithm}[ht]
\caption{Training $\alpha$-flow}\label{alg:train}
\begin{algorithmic}[1]
\WHILE{not converged}
    \STATE Sample $\mu_1\sim q(\mu),\mu_0\sim p_0(\mu)$.
    \STATE Apply $x_1=\pi^{(\alpha)}(\mu_1),x_0=\pi^{(\alpha)}(\mu_0)$ in Eq.\ref{eqn:mapping}.
    \STATE Sample $t\sim U[0,1]$ and calculate the mapped geodesic $x_t=\widetilde{\exp}_{x_0}^{(\alpha)}\big(t\widetilde{\log}_{x_0}^{(\alpha)}(x_1)\big)$ using Eq.\ref{eqn:mapped_exp}, \ref{eqn:mapped_log}.
    \STATE Calculate the conditional vector field $u_t=\widetilde{\log}_{x_t}^{(\alpha)}(x_1)/(1-t)$.
    \STATE Optimize the loss $\|v_\theta(x_t,t)-u_t\|_\alpha^2$ in Eq.\ref{eqn:mapped_norm}.
\ENDWHILE
\end{algorithmic}
\end{algorithm}

\subsection{Variational Bound for Negative Log-Likelihood}\label{sec:elbo}
The conditional flow matching framework enjoys the exact likelihood calculation. However, as discussed in previous work \citep{cheng2024categorical}, such a likelihood is \emph{continuous} and non-comparable to the \emph{discrete} categorical likelihood of $\mu$. In this subsection, we demonstrate that, similar to discrete-state DFM models, a variational bound can be established for the discrete negative log-likelihood (NLL) for all $\alpha$-flows.

\begin{theorem}[Negative ELBO]\label{thm:elbo}
For all $\alpha\in[-1,1]$, the loss $\mathcal{L}^{(\alpha)}$ establishes a negative evidence lower bound (ELBO) for the discrete negative log-likelihood:
\begin{equation}
    -\log p(\delta_1)\le \frac{1}{2}\mathbb{E}_{t,p_0(\mu)}\|v_\theta(\gamma^{(\alpha)}(t),t)-\dot{\gamma}^{(\alpha)}(t)\|_g^2+C
\end{equation}
where $\delta_1$ is the target one-hot distribution, $\gamma^{(\alpha)}$ is the $\alpha$-geodesic connecting $\mu_0$ to $\delta_1$, and $C$ is a non-negative constant that does not rely on the model parameter $\theta$.
\end{theorem}

The following proposition of the property of KL-divergence is necessary. See Appendix~\ref{supp:proof_div_taylor} for proof.
\begin{proposition}\label{prop:div_taylor}
For $\mu\in\mathcal{P}_+$,
\begin{equation}
D_\text{\rm KL}(\mu\|\mu+\mathrm{d}\mu)=\frac{1}{2}\|\mathrm{d}\mu\|_g^2 +o\left(\|\mathrm{d}\mu\|_g^2\right) .
\end{equation}
\end{proposition}

\begin{proof}[Proof for Theorem~\ref{thm:elbo}]
Following \citet{lipman2024flow}, a variational bound for the discrete log-likelihood for the continuous Markov process can be established as:
\begin{equation}
\begin{aligned}
    \mathcal{L} &=\mathbb{E}_{p_0(\mu)}\Big[-\log p_\theta(\delta_1|\mu_1)\\
    &+\int_0^1 \frac{1}{\mathrm{d}t}D_\text{KL}\big(q(\delta_{t+\mathrm{d} t}|\mu_t,\delta)\|p_\theta(\delta_{t+\mathrm{d}t}|\mu_t)\big)\\
    &+D_\text{KL}\left(q(\mu_0|\delta_1)\|p_\theta(\mu_0)\right)\Big] \ge -\log p(\delta_1),
\end{aligned}
\end{equation}
where $\delta_t\sim \mu_t$ is the intermediate one-hot distribution. Here, $q$ denotes the categorical distribution generated by the true denoising process, and $p_\theta$ denotes the predicted categorical distribution (in the infinitesimal sense). The three components are also known as the reconstruction, diffusion, and prior losses, respectively.
For $\alpha$-flow, we have $\mu_1\equiv\delta_1$ as $\tau_{\mu,\nu}(1)=1$ in the interpolation. Therefore, the reconstruction loss is always 0.
As $p_0(\mu)$ is independent of $\delta_1$, the prior loss is a non-negative constant $C$ that does not rely on the model parameter $\theta$.

Note that since the probability paths of $q$ are defined as the $\alpha$-geodesics, we have $q(\delta_{t+\mathrm{d}t}|\mu_t,\delta_1)=\gamma^{(\alpha)}(t+\mathrm{d}t)$. By definition, such a conditional probability changes at a rate of the ground truth vector field $\dot\gamma^{(\alpha)}(t)$.
For an infinitesimal timestep $\diff t$, the integrant of $D_\text{KL}$ can be approximated as:
\begin{align}
    &\quad\;D_\text{KL}\big(q(\delta_{t+\mathrm{d}t}|\mu_t,\delta)\|p_\theta(\delta_{t+\mathrm{d}t}|\mu_t)\big)\nonumber\\
    &= D_\text{KL}\left(\exp_{\mu_t}^{(\alpha)}\left(\dot{\gamma}^{(\alpha)}(t)\diff t\right)\|\exp_{\mu_t}^{(\alpha)}\left(v_\theta(\mu_t,t)\diff t\right)\right)\nonumber\\
    &\approx D_\text{KL}\left(\mu_t+\dot{\gamma}^{(\alpha)}(t)\diff t\|\mu_t+v_\theta(\mu_t,t)\diff t\right)\nonumber\\
    &\approx\frac{1}{2}\|\dot{\gamma}^{(\alpha)}(t)-v_\theta(\mu_t,t)\|_g^2\diff t^2.
\end{align}
We use the first-order Taylor expansion of the exponential map as $\exp^{(\alpha)}_\mu(a)\approx\mu+a$ for the first approximation and Proposition~\ref{prop:div_taylor} for the second. Therefore, integrating over $t$, we get
\begin{equation}
    \mathcal{L}=\frac{1}{2}\int_0^1 \|\dot{\gamma}^{(\alpha)}(t)-v_\theta(\mu_t,t)\|_g^2 \diff t+C=\frac{1}{2}\mathcal{L}^{(\alpha)}+C. 
\end{equation}
Marginalization over $\mu_0\sim p_0(\mu)$ concludes the proof.
\end{proof}

In this way, when optimizing the loss in Eq.\ref{eqn:loss}, it is equivalent to optimizing the negative ELBO for the discrete likelihood, as the gradient of the loss is the same.

\subsection{Optimality of $\alpha$-Flow}
With the unified theoretical framework, we proceed to demonstrate the optimality of $\alpha$-flow from an optimization perspective.

\begin{definition}[$\alpha$-divergence] \label{def:alpha_div}
The following function $D^{(\alpha)}:\mathcal{P}_+\times \mathcal{P}_+\to\mathbb{R}_+$ is called the $\alpha$-\emph{divergence}:
\begin{equation}
    D^{(\alpha)}(\mu\|\nu):=\frac{4}{1-\alpha^2}\sum_{i=1}^n\nu_i^{\frac{1+\alpha}{2}}\mu_i^{\frac{1-\alpha}{2}} .
\end{equation}
\end{definition}
The $\alpha$-divergence can be naturally continued at $\alpha=\pm 1$ by taking the limit to obtain $D^{(-1)}(\mu\|\nu)=D_\text{KL}(\mu\|\nu)$ and $D^{(1)}(\mu\|\nu)=D_\text{KL}(\nu\|\mu)$. 

\begin{proposition}\label{prop:gradient}
The following holds:
\begin{equation}
    \log^{(\alpha)}_\mu \nu = \dot\tau_{\mu,\nu}(0)P_\mu\left(-\grad_\mu D^{(-\alpha)}(\cdot\|\nu)\right).
\end{equation} 
\end{proposition}
See Appendix~\ref{supp:proof_grad} for proof. In other words, the vector field direction along the $\alpha$-geodesic coincides with the projection of the negative gradient of the $(-\alpha)$-divergence, the steepest direction of decreasing $(-\alpha)$-divergence.
For general non-Levi-Civita connections, the local length-minimizing properties of the geodesics are no longer valid. However, a generalization of the statement is available:
\begin{theorem}[local optimality, \citet{bauer2024p}, Corollary 3.11]\label{thm:local_optim}
The $\alpha$-geodesic $\gamma^{(\alpha)}$ describes locally minimizing curves of the (local) $\alpha$-\emph{energy} defined as
\begin{equation}
    E^{(\alpha)}(\mu):=\frac{1}{p}\int_0^1 F^p(\gamma,\dot{\gamma}) \diff t \label{eqn:alpha_energy}
\end{equation}
for any locally defined smooth curve $\gamma$ starting at $\mu$, where $F(\gamma,\dot{\gamma})$ is the Finsler metric
\begin{equation}
    F(\mu,a):=\left(\sum_{i=1}^n\left|\frac{a_i}{\mu_i}\right|^p\mu_i\right)^{1/p} .
\end{equation}
\end{theorem}

As a special case of the Levi-Civita connection when $\alpha=0,p=2$, the Finsler metric coincides with the Fisher information metric, and we obtain:
\begin{corollary}
The 0-geodesic locally minimizes the curve length and the kinetic energy $\frac{1}{2}\int_0^1\|\dot{\gamma}\|_g^2\diff t$.
\end{corollary}

Furthermore, as we confine ourselves to the positive orthant of the $L_p$ sphere, the geodesic is unique (\citet{bauer2024p}, Theorem 4.2), leading to global optimality:
\begin{theorem}[global optimality]\label{thm:global_optim}
The $\alpha$-geodesic $\gamma^{(\alpha)}$ describes curves that minimize the (global) $\alpha$-\emph{energy} $E^{(\alpha)}(\mu,\nu):=\frac{1}{p}\int_0^1 F^p(\gamma,\dot{\gamma}) \diff t$ among any smooth curve $\gamma$ connect $\mu$ to $\nu$.
\end{theorem}

\begin{corollary}
The 0-geodesic globally minimizes the curve length and the kinetic energy $\frac{1}{2}\int_0^1\|\dot{\gamma}\|_g^2\diff t$.
\end{corollary}

As the Fisher information metric coincides with the standard sphere metric for $\alpha=0$, the mapped geodesics are great circles on the sphere, leading to the following corollaries:
\begin{corollary}\label{coro:sphere_geodesic}
For the vector field $u:=a/\sqrt{\mu}$ on the sphere, the curve that minimizes the kinetic energy $\frac{1}{2}\int_0^1 \|u_t\|^2_2\diff t=\frac{1}{2}\int_0^1 \|a_t\|^2_g\diff t$ between two points $x=\sqrt{\mu},y=\sqrt{\nu}$ is the great circle connecting $x$ and $y$.
\end{corollary}
\begin{corollary}
For two one-hot distributions $\delta_0,\delta_1$, the interpolation along the kinetic-energy minimizing curve adopts the following form:
\begin{equation}
    \mu_t=\delta_0\cos^2(t\pi/2)+\delta_1\sin^2(t\pi/2) .
\end{equation}
\end{corollary}
The above corollaries coincide with the core results in \citet{shaul2024flow} and justify the usage of the cosine-squared scheduler $\kappa(t)=\cos^2(t\pi/2)$ in \citet{han2022ssd}.

\subsection{Connection to Existing DFM Models}\label{sec:classification}
\input{tabs/classification}
In this subsection, we systematically discuss how existing discrete flow matching models are connected to our unified framework and how they may assume different geometries from our $\alpha$-flow.
Following the conventions in information geometry, the ($-1$)-geometry is called the $m$-geometry ($m$ for mixture), and the 1-geometry is called the $e$-geometry ($e$ for exponential). 
Despite our focus on continuous-state DFM, we note the potential of applying the concept of $\alpha$-geometry in discrete-state DFM (DS-DFM) models, for which a probability path of categorical distributions must also be defined.
We summarize the different classes of $\alpha$-geometries together with each related existing work in Table~\ref{tab:class}.

\textbf{Mixture class ($\alpha=-1$).} 
The $m$-geometry on the statistical manifold admits the identity embedding $\mu\mapsto \mu$ with a flat structure. The exponential and logarithm maps share the same formulae as the common Euclidean geometry.
We generally consider any model that operates linearly on the categorical probabilities (i.e., $m$-representations) as \emph{mixture-class}. This includes linear flow matching used in previous work \citep{stark2024dirichlet,cheng2024categorical}. The only difference is that the loss for linear flow matching uses the Euclidean norm $\|\cdot\|_2^2$ instead of the Riemannian norm $\|\cdot\|_g^2$. \citet{song2023equivariant,dunn2024mixed} also used such linear models to generate discrete modality.
Furthermore, most recent work of DS-DFM models \citep{sahoo2024simple,shi2024simplified,gat2024discrete} primarily adopted the $m$-geodesic as the interpolation path between the mask and target tokens and can also be classified as the mixture class.

\textbf{Metric class ($\alpha=0$).} 
The mapping $\pi^{(0)}:\mu\mapsto \sqrt{\mu}$ projects categorical distributions onto the unit sphere whose induced metric coincides with the inherited Euclidean metric (the canonical spherical geometry) up to a constant scaling factor. Therefore, in this case, the $\alpha$-geodesic is the Levi-Civita geodesic that enjoys additional theoretical benefits. With the spherical symmetry, the reparameterizations can be solved in closed form (see Appendix~\ref{supp:solve_geodesic}), leading to the canonical spherical geometry with the (mapped) $0$-geodesics being great circles on the sphere.

Two previous works \citep{cheng2024categorical,davis2024fisher} have explored such a canonical Riemannian structure on the statistical manifold of categorical distributions, following the spherical geodesics between the mapped representations of categorical distributions. Both papers also demonstrated the close relation of such a metric-class generative model to \emph{natural gradient} \citep{amari1998natural,amari1998adaptive} as the ``steepest'' direction of decreasing KL-divergence. As noted in Proposition~\ref{prop:gradient}, a similar property holds for all variants of $\alpha$-flow.
Additionally, in DS-DFM models, \citet{shaul2024flow} adopted a kinetic-optimal perspective to construct the path of categorical distributions. When using the canonical Fisher information metric, \citet{shaul2024flow} arrived at exactly the same result as our Corollary \ref{coro:sphere_geodesic}, which is a special case of Theorem \ref{thm:local_optim} and \ref{thm:global_optim} establishing optimality for $\alpha$-flow.
The usage of the corresponding scheduler (the cosine-squared scheduler) was proposed in \citet{han2022ssd} purely heuristically. Our results provide additional theoretical justifications.

\textbf{Exponential class ($\alpha=1$).} 
The $e$-geometry on the statistical manifold admits an almost linear structure on the logit space with the logit mapping $\mu\mapsto\log\mu$ with an almost-linear exponential and logarithm maps shown in Table~\ref{tab:class}.
To the best of our knowledge, such exponential and logarithm maps have not been proposed in previous work. \citet{boll2024generative,boll2024generative2} explored a similar setting along the $e$-geodesic of the assignment manifold but used a different parameterization.
Similar to the previous two classes, we generally consider any models operating on the logits (i.e., the $e$-representation up to the equivalence in the affine structure) as \emph{exponential-class}. In this sense, the earliest attempts at extending flow matching and diffusion models to discrete generative modeling can be considered. For example, \citet{hoogeboom2021argmax} assumes multiplicative noising of probabilities, therefore linearly operating in the logits space. \citet{mahabadi2023tess} performed diffusion on the logit space, and \citet{li2024full} adopted the linear flow matching on the logits for the discrete modality. 
These existing models usually directly assume an Euclidean structure on the logits without the additional correction term. Our exponential-class model, on the other hand, is the rigorous mathematical result from information geometry.

%% file: tabs/classification.tex
\begin{table*}[htb]
\centering
\caption{Classes of $\alpha$-geometry and their connections to existing DFM methods. Though operating on the corresponding $\alpha$-representation, many of the existing works assumed Euclidean geometry, which could fail to capture the underlying $\alpha$-geometry.}\label{tab:class}
\vspace{-0.5em}
\small
\setlength{\abovedisplayskip}{2pt} 
\setlength{\belowdisplayskip}{2pt} 
\setlength{\jot}{0pt}
\resizebox{\linewidth}{!}{
\begin{tabular}{@{}lcccc@{}}
\toprule
Class & $\alpha$-Representation & $\alpha$-Geometry & CS-DFM & DS-DFM \\ \midrule
\begin{tabular}[c]{@{}l@{}}Mixture\\ ($\alpha=-1$)\end{tabular} & $x=\mu$ & 
\parbox{5cm}{
\begin{align*}
\widetilde{\exp}^{(-1)}_x u&=x+u\\
\widetilde{\log}^{(-1)}_x y&=y-x
\end{align*}
}
& \begin{tabular}[c]{@{}c@{}}LinearFM [\citenum{lipman2022flow,stark2024dirichlet}]\end{tabular}
& \begin{tabular}[c]{@{}c@{}}MDLM [\citenum{sahoo2024simple}]\\ MD4 [\citenum{shi2024simplified}] \\ DFM [\citenum{gat2024discrete}]\end{tabular} \\
\begin{tabular}[c]{@{}l@{}}Metric\\ ($\alpha=0$)\end{tabular} & $x=\sqrt\mu$ & 
\parbox{5cm}{
\begin{align*}
\widetilde{\exp}^{(0)}_x u&=x\cos\|u\|_2+\frac{u}{\|u\|_2}\sin\|u\|_2\\
\widetilde{\log}^{(0)}_x y&=\frac{\arccos(\langle x,y\rangle)}{\sqrt{1-\langle x,y\rangle^2}}(y-x-\langle x,y-x\rangle x)
\end{align*}
}
& \begin{tabular}[c]{@{}c@{}}SFM [\citenum{cheng2024categorical}]\\ FisherFlow [\citenum{davis2024fisher}]\end{tabular}
& \begin{tabular}[c]{@{}c@{}}DFM-KO [\citenum{shaul2024flow}] \\ $\cos^2$ scheduler [\citenum{han2022ssd}]\end{tabular} \\
\begin{tabular}[c]{@{}l@{}}Exponential\\ ($\alpha=1$)\end{tabular} & $x=\log \mu$ &
\parbox{5cm}{
\begin{align*}
\widetilde{\exp}^{(1)}_x u&=x+u-\logsumexp(x+u) \\
\widetilde{\log}^{(1)}_x y&=y-x+D_\text{KL}(\mu\|\nu)
\end{align*}
}
& \begin{tabular}[c]{@{}c@{}}AssignmentFlow [\citenum{boll2024generative,boll2024generative2}]\\ ArgmaxFlow [\citenum{hoogeboom2021argmax}] \\ TESS [\citenum{mahabadi2023tess}]\end{tabular}
& None \\ \bottomrule
\end{tabular}
}
\end{table*}

%% file: secs/4_experiment.tex
\section{Experimental Result}
To thoroughly study the different initializations of our proposed $\alpha$-flow, we systematically conduct evaluations across various discrete generative tasks in the domains of computer vision, natural language processing, and bioinformatics.
In addition to the three existing classes of $\alpha$-flow that enjoy closed-form solutions of the geodesic reparameterization ($\boldsymbol{\alpha=\pm 1,0}$), we also numerically solved for $\boldsymbol{\alpha=\pm 0.5}$ (see Appendix~\ref{supp:solve_geodesic}) to provide finer-grained tuning results on $\alpha$.
\textbf{Linear} represents the linear flow matching model often used as a baseline in previous work.
Additional baselines of discrete-state DFM models are also compared. We include \textbf{MDLM} \citep{sahoo2024simple} and \textbf{DFM} \citep{gat2024discrete} as state-of-the-art DS-DFM models for comparison. For MDLM and DFM, the final predictions are directly interpreted as logits to stochastically unmask each token. 

\subsection{Toy Example: Swiss Roll}

Following \citet{cheng2024categorical}, the toy example of the Swiss roll on the 2-simplex serves as the proof-of-concept of the effectiveness of $\alpha$-flow. We sampled 10k points as the training set and generated 10k points from each trained model using 1000 Euler steps. We used Gaussian kernel density estimation (KDE) to evaluate the generation quality to approximate the density and calculate the KL divergence between the training data and the generation. 
As the outputs for DS-DFM models are discrete tokens, they cannot be applied to this task. The sample KL divergences and estimated kernel densities are shown in Figure~\ref{fig:swissroll}. All CS-DFM models could capture this complex geometry on the simplex well, demonstrating $\alpha$-flow as a continuous family of CS-DFM models.

\begin{figure}[htb]
    \centering
    \includegraphics[width=.8\linewidth]{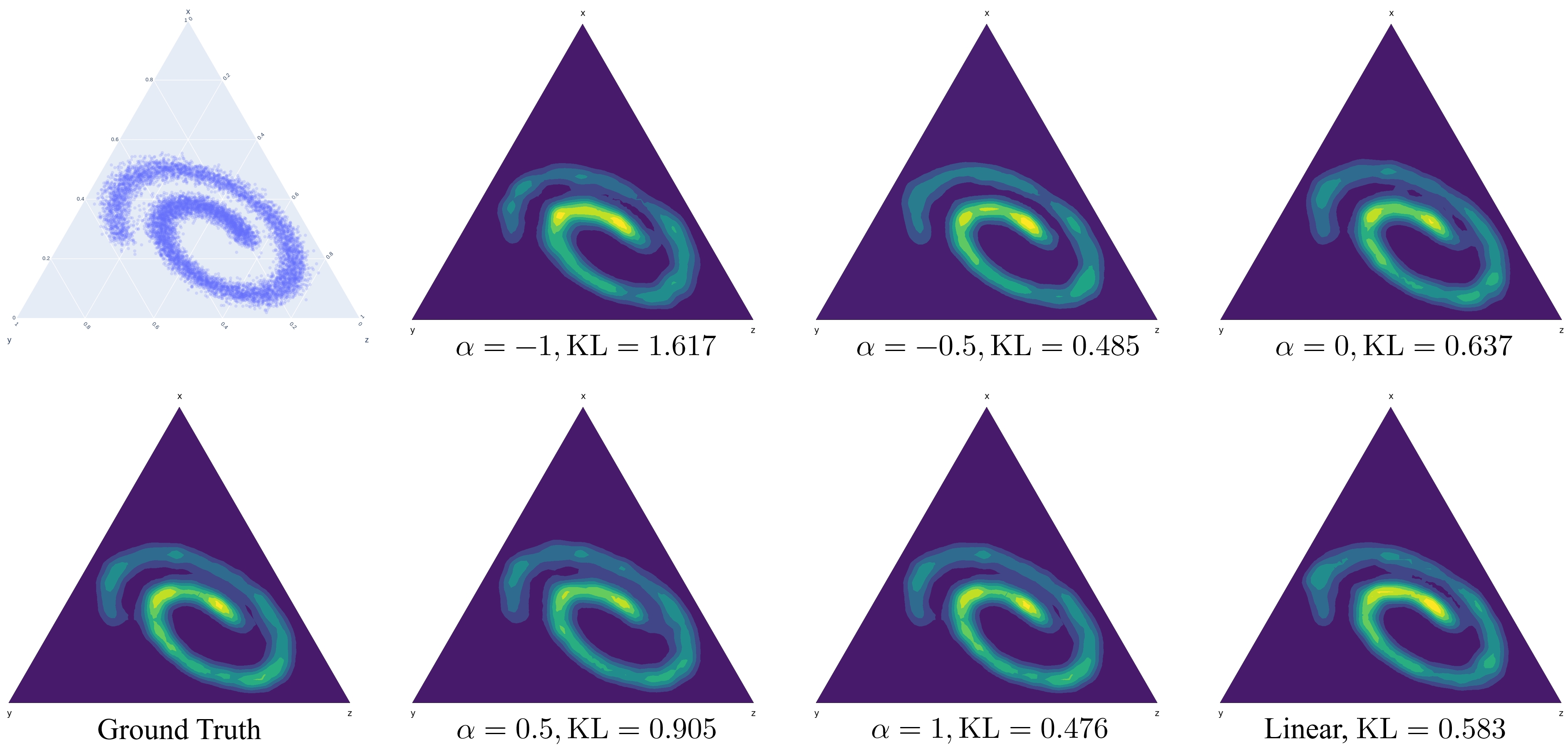}
    \vspace{-.5em}
    \caption{Estimated densities using different variants of $\alpha$-flow and KL divergence to the ground truth density estimation.}
    \label{fig:swissroll}
\end{figure}

\subsection{Image Generation: Binarized MNIST}
The binarized MNIST dataset \citep{salakhutdinov2008quantitative} is the binarized version of MNIST \citep{lecun2010mnist} by thresholding the original values to be either 0 or 1, leading to a 2-class discrete generative modeling task with a data dimension of $28^2=784$. All models use the same convolutional neural network (CNN) based predictor adapted from \citet{song2020improved} with additional sinusoidal time embeddings.

\begin{figure}[ht]
    \centering
    \includegraphics[width=\linewidth]{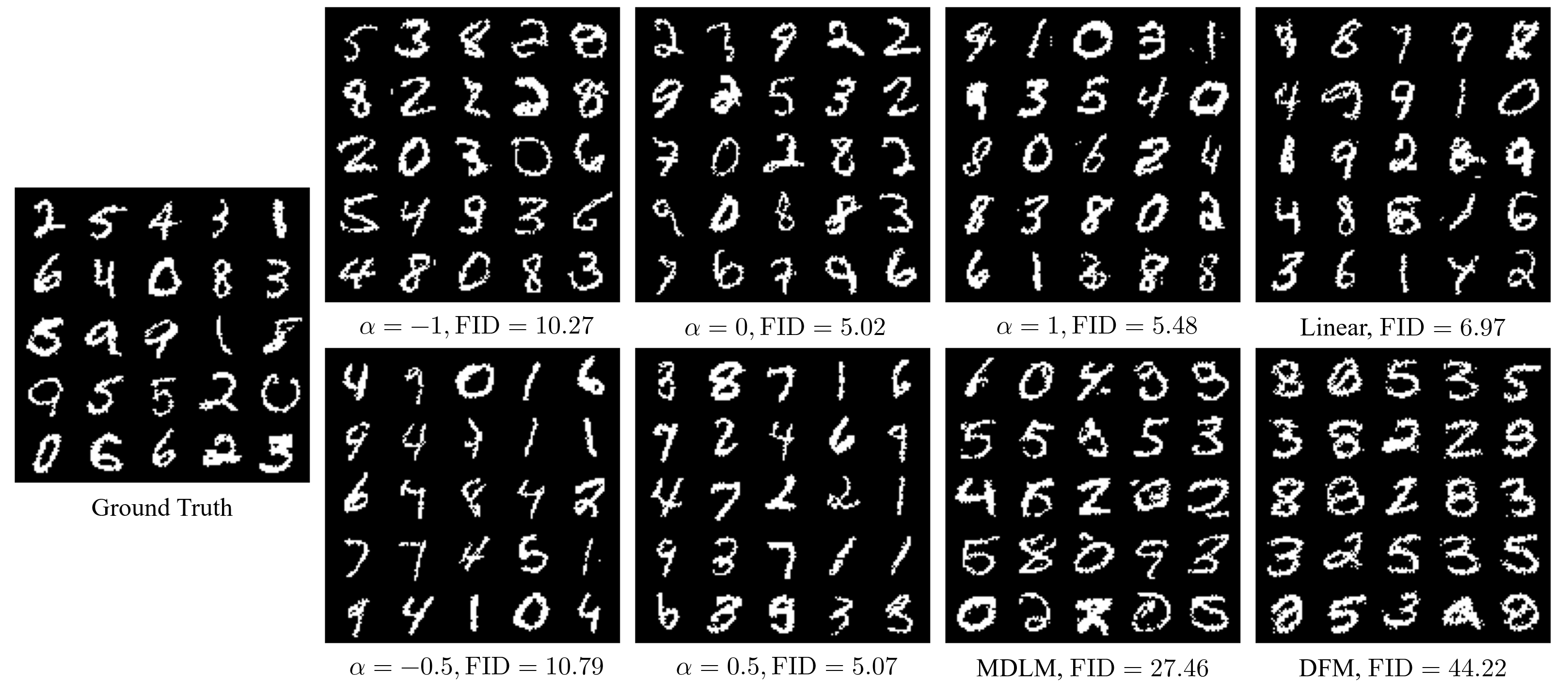}
    \vspace{-1em}
    \caption{Uncurated generated digits and FID scores (lower is better) on the binarized MNIST dataset.}
    \label{fig:bmnist}
    \vspace{-1em}
\end{figure}

To evaluate the generation quality quantitatively, we follow \citet{cheng2024categorical} to calculate the Fréchet Inception distance (FID) to reflect the distributional fitness. For each model, 1000 images were sampled using 300 Euler steps.
Figure~\ref{fig:bmnist} provides the FID scores with uncurated generated samples for different models, where all CS-DFM models consistently and significantly outperformed DS-DFM models. Among the $\alpha$-flow variants, the mixture class ($\alpha=-1$ and Linear) falls short compared to the metric and exponential classes, supporting the curved geometry of the statistical manifold. It is worth noting that $\alpha=0.5$ achieved almost the same best performance as $\alpha=0$, demonstrating the effectiveness of intermediate $\alpha$ values that have not been explored in previous work.

\subsection{Language Modeling: Text8}
The Text8 dataset \citep{mahoney2011large} is a medium-sized character-level corpus with a total number of 100M characters and a small vocabulary of 27. We follow previous work \citep{austin2021structured,campbell2024generative} to use a split of 90M/5M/5M and random chunks of length 256 without any preprocessing. Following \citet{cheng2024categorical}, we used a 12-layer diffusion transformer (DiT) \citep{peebles2023scalable} based predictor for all models and trained each model with a total batch size of 1M characters and a total number of 500B characters. 
Following \citet{campbell2024generative}, we utilize the pre-trained GPT-J-6B \citep{gpt-j} model to calculate the NLL of the generations as a model-agnostic evaluation metric for both CS- and DS-DFM models. Entropy on the tokenized generations is also calculated to reflect the resemblance to the ground truth data. The NLLs and entropies were calculated over a total number of 1M characters for each trained model using 256 Euler steps. The results are summarized in Table \ref{tab:text8}. 

\input{tabs/text8_uniref}

Similar to previous tasks, the metric- and exponential-class models achieved better NLLs than the mix-class models among CS-DFM.
We noted that the CS-DFM models captured the entropy more accurately (i.e., diversity), whereas the DS-DFM models achieved better NLLs (i.e., consistency). As noted in previous work \citep{campbell2024generative}, such a proxy NLL can be fooled into arbitrarily low values by repeating high-frequency but meaningless words. $\alpha$-flow, on the other hand, offers a more flexible trade-off between the NLL and entropy, with $\alpha=0$ achieving better NLL and $\alpha=\pm 1,0.5$ achieving better entropy. This indicates CS-DFM models are more resistant to hallucination \citep{aithal2024understanding} than DS-DFM models, offering better distributional fitness to the ground truth dataset.  

\subsection{Protein Design: UniRef50}

The UniRef50 dataset \citep{suzek2015uniref} consists of about 45 million protein sequences, totaling about 14B amino acid tokens. Following previous work \citep{alamdari2023protein,wang2024diffusion}, we use a vocabulary size of 33 and a maximum sequence length of 1024, and chunk longer proteins into shorter parts. We used a 12-layer DiT-based predictor for all models and trained each model with a total batch size of 800k tokens and a total number of 160B tokens.
To quantitatively evaluate the generative models, we follow a similar approach in \citet{frey2023protein} to calculate the Fréchet distance between the hidden activations of the data and generations extracted by the pre-trained ESM-2 model \citep{lin2023evolutionary}, which we term Fréchet ESM distance (FED). A lower FED score indicates better distributional fitness in terms of data diversity. Additionally, we also follow \citet{wang2024diffusion,wang2024dplm} to calculate the predicted local distance difference test (pLDDT) score using the ESMFold model \citep{lin2023evolutionary} to reflect the foldability (likelihood) of the generated protein sequences. For each model, 100 sequences of each length of $[100,200,300,400,500,1000]$ were generated using 256 Euler steps for calculating the statistics. The results are summarized in Table~\ref{tab:uniref}.

\input{tabs/uniref}

Again, CS-DFM models significantly outperformed DS-DFM models on this task, especially regarding the pLDDT scores. Interestingly, the $\alpha=0.5$ variant achieved the best performance, with the highest foldability as well as decent FED scores.
In Figure~\ref{fig:uniref_alpha}, we further observe how $\alpha$ impacts the subtle trade-off between consistency (pLDDT) and diversity (FED). For example, the metric-class model ($\alpha=0$) consistently maintains the consistency with the lowest FED scores, whereas $\alpha=\pm 0.5$ tends to sacrifice diversity for consistency on shorter sequences and maintains both properties for longer sequences.
Our findings suggest that tuning $\alpha$ has a significant impact on the model's generalization, which further necessitates the exploration of intermediate $\alpha$ values that have not been explored in previous work. Indeed, as a continuous family of generative models, $\alpha$-flow offers a broader architecture design space that can potentially lead to better performance, as demonstrated by our results on UniRef50.


\subsection{Discussion on $\alpha$-Flow}
With extensive experiments of different instantiations of $\alpha$-flow and baseline DS-DFM models across diverse domains, we now discuss their performance and scaling behaviors. 
We first note that the metric- and exponential-class models often have better generative performance than the mix-class models, indicating that the naive flat geometry may fail to capture the complex data manifold.
Additionally, we have demonstrated that intermediate $\alpha$ values can also achieve decent and even better performance, which further suggests the necessity for $\alpha$-tuning for different tasks and metrics. We show that $\alpha$ has a significant impact on the trade-off between consistency (likelihood) and diversity (fitness) in a dataset-dependent fashion. This can be partially understood as the delicate tradeoff is not often homogeneous across domains. Nevertheless, such heterogeneity does not affect the meaningfulness of our proposed design space of continuously choosing $\alpha\in[-1,1]$, given its potential for upside in several tasks. Indeed, compared to DS-DFM models, our $\alpha$-flow offers a more flexible approach that allows for tuning such a trade-off, leading to potentially the optimal geometry in a task-specific manner.

%% file: tabs/text8_uniref.tex
\begin{table*}[htb]
\centering
\begin{minipage}{.58\textwidth}
\centering
\vspace{-0.5em}
\caption{GPT-J-6B evaluated NLLs and generation entropy on Text8. The best is in \textbf{bold}, and the second best is \underline{underlined}. \textsuperscript{*} indicates results from \citet{campbell2024generative}.}\label{tab:text8}
\small
\resizebox{\linewidth}{!}{
\begin{tabular}{@{}llccc@{}}
\toprule
Model Type & Model & NLL & Entropy & Entropy Diff \\ \midrule
\multirow{6}{*}{CS-DFM} & $\alpha=-1$ & 7.31 & 7.45 & \textbf{-0.03} \\
 & $\alpha=-0.5$ & 7.14 & 7.37 & -0.11 \\
 & $\alpha=0$ & 6.85 & 7.38 & -0.10 \\
 & $\alpha=0.5$ & 7.00 & 7.42 & {\ul -0.06} \\
 & $\alpha=1$ & 7.09 & 7.51 & \textbf{+0.03} \\
 & Linear & 7.35 & 7.62 & +0.14 \\ \midrule
\multirow{5}{*}{DS-DFM} & MDLM & 6.76 & 7.55 & +0.07 \\
 & DFM & 6.78 & 7.58 & +0.10 \\
 & D3PM[\citenum{austin2021structured}]\textsuperscript{*} & 6.93 & 7.38 & -0.10 \\
 & SEDD[\citenum{lou2023discrete}]\textsuperscript{*}& \textbf{6.49} & 7.17 & -0.31 \\
 & MultiFlow[\citenum{campbell2024generative}]\textsuperscript{*} & {\ul 6.73} & 7.39 & -0.09 \\ \midrule
 & Data\textsuperscript{*} & 4.10 & 7.48 & 0 \\ \bottomrule
\end{tabular}
}
\end{minipage}%
\hfill
\begin{minipage}{0.41\textwidth}
\centering
\includegraphics[width=\linewidth]{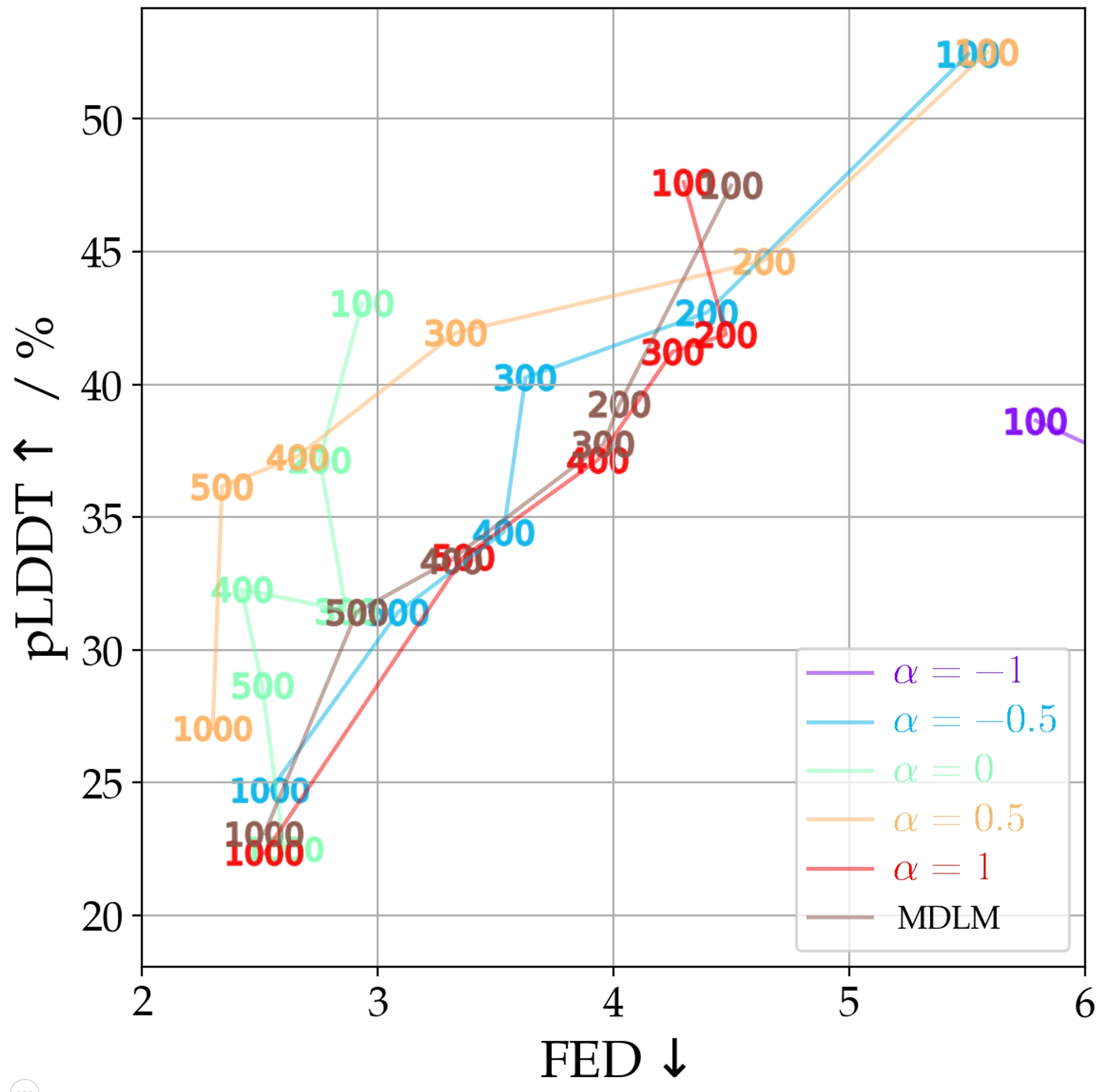}
\captionof{figure}{pLDDT vs FED scores for different variants of $\alpha$-flow and MDLM (the best DS-DFM model).}
\label{fig:uniref_alpha}
\end{minipage}
\end{table*}

%% file: tabs/uniref.tex
\begin{table}[ht]
\vspace{-0.5em}
\centering
\caption{FED and pLDDT scores for generated protein sequences of various lengths using different models on the UniRef50 dataset. The best is in \textbf{bold}, and the second best is \underline{underlined}.}\label{tab:uniref}
\small
\begin{tabular}{@{}llcccccccc@{}}
\toprule
Model & Len & $\alpha=-1$ & $\alpha=-0.5$ & $\alpha=0$ & $\alpha=0.5$ & $\alpha=1$ & Linear & MDLM & DFM \\ \midrule
\multirow{6}{*}{\resizebox{1cm}{!}{pLDDT $\uparrow$}} & 100 & 38.64 & {\ul 52.45} & 43.05 & \textbf{52.53} & 47.62 & 39.55 & 47.48 & 35.74 \\
 & 200 & 28.70 & {\ul 42.70} & 37.16 & \textbf{44.65} & 41.87 & 37.96 & 39.26 & 28.42 \\
 & 300 & 25.41 & 40.26 & 31.47 & \textbf{41.95} & {\ul 41.24} & 34.99 & 37.75 & 25.80 \\
 & 400 & 23.79 & 34.44 & 32.28 & \textbf{37.26} & {\ul 37.14} & 32.68 & 33.37 & 23.71 \\
 & 500 & 22.53 & 31.41 & 28.67 & \textbf{36.16} & {\ul 33.50} & 30.72 & 31.43 & 22.08 \\
 & 1000 & 19.71 & {\ul 24.71} & 22.49 & \textbf{27.05} & 22.35 & 22.74 & 23.09 & 18.30 \\ \midrule
\multirow{6}{*}{FED $\downarrow$} & 100 & 5.79 & 5.51 & 2.93 & 5.59 & 4.30 & \textbf{2.66} & 4.50 & {\ul 2.91} \\
 & 200 & 8.24 & 4.40 & \textbf{2.75} & 4.64 & 4.48 & {\ul 2.79} & 4.03 & 4.05 \\
 & 300 & 8.59 & 3.63 & \textbf{2.87} & {\ul 3.33} & 4.25 & \textbf{2.87} & 3.96 & 4.47 \\
 & 400 & 8.46 & 3.53 & \textbf{2.42} & 2.66 & 3.93 & {\ul 2.50} & 3.31 & 4.42 \\
 & 500 & 8.84 & 3.09 & {\ul 2.51} & \textbf{2.34} & 3.36 & 2.57 & 2.91 & 4.38 \\
 & 1000 & 9.76 & 2.54 & 2.63 & \textbf{2.30} & {\ul 2.52} & 2.77 & {\ul 2.52} & 4.76 \\ \bottomrule
\end{tabular}
\end{table}

%% file: secs/5_related.tex
\section{Related Work}
We have mentioned the connections of existing DFM models to $\alpha$-flow in Section~\ref{sec:classification}. Here, we provide a more organized summary of existing work.

\textbf{Continuous-State Discrete Flow Matching.}
The linear flow matching model was compared as a baseline in \citet{stark2024dirichlet,cheng2024categorical}. The mixture-class model that utilizes the Riemannian norm has not been explored in previous work.
The metric-class model coincides with those proposed in \citet{cheng2024categorical,davis2024fisher}.
For the exponential-class model, \citet{boll2024generative,boll2024generative2} used an alternative parameterization on the assignment manifold, which is related to the $e$-geodesic. \citet{hoogeboom2021argmax,mahabadi2023tess,li2024full} operated linearly on the logit space without the manifold constraint on the intermediate steps. 
It is also worth noting that not all CS-DFM models can be classified as operating on some $\alpha$-representation. For example, \citet{avdeyev2023dirichlet} considered probability paths defined by the Jacobi diffusion process, and \citet{stark2024dirichlet} of Dirichlet distributions. 

\textbf{Discrete-State Discrete Flow Matching.}
Generally, we also consider various masked diffusion models and their continuous-time extensions as part of this class. DS-DFM models are characterized by jumps between discrete categories, modeled as the transition between the states of a Markov chain. In DS-DFM models, the reverse diffusion process is parameterized and learned by a neural net to stochastically convert the initial tokens (often masks) to the target tokens along the probability paths. \citet{austin2021structured} adopted the general form of transition matrix and \citet{campbell2024generative} extended it to continuous-time formulation. \citet{lou2023discrete,sahoo2024simple,shi2024simplified} all followed the linear interpolation path between one-hot distributions. \citet{shaul2024flow} adopted a kinetic-optimal perspective to introduce the metric-geodesic as the probability path, which coincides with our $0$-geodesic.

\textbf{Information Geometry in Generative Modeling.}
\citet{cheng2024categorical,davis2024fisher} leveraged the canonical Riemannian geometry on the statistical manifold and noted the relation between the metric-geodesic and the natural gradient. \citet{boll2024generative,boll2024generative2} utilized the assignment manifold related to the $e$-geodesic for flow matching. \citet{kimura2024density} utilized $\alpha$-geodesic for density ratio estimation. To the best of our knowledge, we are the first to adopt $\alpha$-geodesics for discrete flow matching and offer a rigorous and unified mathematical framework. We further establish an ELBO for the discrete NLL and demonstrate optimality for $\alpha$-flow, both of which are crucial for generative models but lack systematic analyses before.

%% file: secs/6_conclusion.tex
\section{Conclusion}
In this work, we present a unified framework for continuous-state DFM models and propose $\alpha$-flow as a mathematically rigorous generative model that adheres to $\alpha$-geometry. Building upon such a theoretical framework, we establish a unified variational bound and demonstrate optimality from the kinetic-optimal perspective for $\alpha$-flow. We conduct comprehensive experiments with different $\alpha$ values across various domains to explore the practical performance of $\alpha$-flow. We observe the impact of $\alpha$ on the trade-off between diversity and consistency, which further necessitates the exploration for $\alpha$-tuning in a task-specific manner.
Built upon such a unified theoretical framework generally applicable to all $\alpha$-flow variants, we hope to inspire new CS-DFM frameworks that offer more flexibility than the DS-DFM models. For example, one may consider choosing $\alpha$ adaptively or consider the mixture of $\alpha$-geodesics.

%% file: secs/7_impact.tex
\section*{Broader Impact}
This work focuses on developing a general theoretical framework for discrete generative modeling. Our $\alpha$-flow offers a generally applicable framework in diverse downstream applications, including image generation, natural language modeling, and protein design. In this way, our framework can potentially positively impact various domains, and we are dedicated to ensuring responsible usage of our framework.

%% file: suppl/A_riemann.tex
\section{Riemannian Geometry}\label{supp:riemannian}
In this section, we briefly introduce the key concepts in Riemannian geometry used in this work. A more comprehensive background on the Riemannian manifold can be found in standard mathematics textbooks like \citet{gallot2004riemannian}.

A Riemannian manifold $\mathcal{M}$ is a real, smooth manifold equipped with a positive definite inner product $g$ on the tangent space $T_x\mathcal{M}$ at each point $x\in\mathcal{M}$. Let $T\mathcal{M}=\bigcup_{x\in\mathcal{M}}T_x\mathcal{M}$ be the \emph{tangent bundle} of the manifold $\mathcal{M}$, a time-dependent \emph{vector field} on $\mathcal{M}$ is a mapping $u_t:[0,1]\times\mathcal{M}\to T\mathcal{M}$ where $u_t(x)\in T_x\mathcal{M}$. 
Let $\Gamma(T\mathcal{M})$ denote the space of vector fields on $\mathcal{M}$, an \emph{affine connection} or \emph{covaraint derivative} 
\begin{equation}
\begin{aligned}
    \nabla:\Gamma(T\mathcal{M})\otimes \Gamma(T\mathcal{M})&\to\Gamma(T\mathcal{M})\\
    (V,W)&\mapsto \nabla_V W
\end{aligned}
\end{equation}
is a bilinear map on $\mathcal{M}$ that
generalizes the idea of directional derivatives in the Euclidean case and allows for the differentiation of vector fields at different points. An affine connection can be defined using its \emph{Christoffel symbols} $\Gamma_{ij}^k$, which are in turn defined via:
\begin{equation}
    \nabla_{\partial x^i}\partial x^j:=\Gamma^k_{ij}\partial x^k
\end{equation}
where $\{\partial x^i\}$ are the local coordinate bases of the vector field. 
A smooth curve $\gamma:[0,1]\to\mathcal{M}$ is called a \emph{geodesic} if it is autoparallel or:
\begin{equation}
    \nabla_{\dot\gamma}\dot\gamma=0. \label{eqn:geodesic_eqn}
\end{equation}
Intuitively, the vector field $\dot\gamma$ along the curve remains ``constant'' as evaluated by the affine connection. The geodesic equation in Eq.\ref{eqn:geodesic_eqn} can be written locally in terms of Christoffel symbols using the Einstein summation convention as:
\begin{equation}
    \ddot{\gamma}^k+\Gamma^k_{ij}\dot\gamma^i\dot\gamma^j=0,\quad k=1,\dots,n .
\end{equation}
A \emph{Levi-Civita connection} is the unique affine connection that is torsion-free and metric compatible $\nabla g=0$. As the Levi-Civita connection preserves the metric (and, therefore, preserves the length), we arrive at the standard definition of geodesic as a locally distance-minimizing curve on the manifold. However, this property does not always hold for general geodesics.

The existence and the uniqueness of the geodesic state that for any point $x\in\mathcal{M}$ and for any tangent vector $u\in T_x\mathcal{M}$, there exists a unique geodesic $\gamma:[0,1]\to\mathcal{M}$ such that $\gamma(0)=x$ and $\dot\gamma(0)=u$. The \emph{exponential map} $\exp:\mathcal{M}\times T\mathcal{M}\to\mathcal{M}$ is uniquely defined to be $\exp_x(u):=\gamma(1)$. The \emph{logarithm map} $\log:\mathcal{M}\times \mathcal{M}\to T\mathcal{M}$ is defined as the inverse mapping of the exponential map such that $\exp_x(\log_x(y))\equiv y,\forall x,y\in \mathcal{M}$. With the exponential map and logarithm map, the time-dependent flow can be compactly written as time interpolation along the geodesic:
\begin{equation}
    x_t:=\gamma_{x,y}(t)=\exp_{x}(t\log_{x}y),\quad t\in[0,1] .
\end{equation}

%% file: suppl/B_infogeo.tex
\section{Information Geometry}
In this section, we give a more comprehensive introduction to information geometry and the geometry of the statistical manifold of categorical distributions. Information geometry studies statistical manifolds, which are Riemannian manifolds whose points correspond to probability distributions. It can be traced back to C. R. Rao's seminal work in 1945 \citep{rao1992information} studying the Fisher information metric. Modern advances can be found in various mathematical textbooks of \citet{amari2000methods,amari2016information,ay2017information}.

\subsection{Statistical Manifolds}
The \emph{statistical manifold} can be intuitively understood as the manifold of a family of parameterized probability distributions. More rigorously, for a sample space $\mathcal{X}$ and its reference measure $\nu$, we consider the parameterized probability distributions defined using the Radon-Nikodym derivative $\mathcal{P}(\mathcal{X})=\{p_\theta:\int_\mathcal{X}p(x;\theta)\diff\nu=1\}$. The parameterization $\theta\in\Theta$ gives a natural local chart on the statistical manifold. The Fisher information metric provides a canonical way of equipping the statistical manifold with a Riemannian structure and is defined as
\begin{equation}
    g_{ij}(\theta)=\mathbb{E}_X\left[\frac{\partial\log p(X;\theta)}{\partial \theta^i}\frac{\partial\log p(X;\theta)}{\partial \theta^j}\right]=\int_\mathcal{X}\frac{\partial\log p(x;\theta)}{\partial \theta^i}\frac{\partial\log p(x;\theta)}{\partial \theta^j} p(x;\theta)\diff\nu . \label{eqn:fisher}
\end{equation}

As we primarily focus on discrete generative modeling, we will consider the discrete sample space $\mathcal{X}=\{1,2,\dots,n\}$ with the canonical counting measure as the reference measure. In this way, any probability distribution can be represented as a convex combination of Dirac measures $\delta^i$ as $\mu=\mu_i\delta^i,\sum_{i=1}^n\mu_i=1,0\le \mu_i\le 1$. Therefore, we can visualize such a statistical manifold as a probability simplex. Direct calculation gives the Fisher information metric as
\begin{equation}
    g_{ij}(\mu)=\frac{\delta_{ij}}{\mu_i}+\frac{1}{\mu_n},\quad 1\le i,j,\le n-1,
\end{equation}
and leads to the Riemannian inner product in Eq.\ref{eqn:inner}.

\subsection{$\alpha$-Representation for Positive Measures} \label{supp:alpha_rep_m}
The traditional way of obtaining affine connections on the statistical manifold starts with the definitions of the $m$- and $e$-connections on the manifold of positive measures $\mathcal{M}_+=\mathcal{M}_+(\mathcal{X})$ and then projects them onto $\mathcal{P}_+\subseteq\mathcal{M}_+$. We will briefly introduce the $\alpha$-geometry for positive measure $m\in\mathcal{M}_+$ for completeness of the mathematical background. We refer interested readers to mathematical textbooks like Chapter 2.5.2 in \citet{ay2017information}.

The $\alpha$-representation of a positive measure $m$ adopts the identical formula as Definition~\ref{def:alpha_rep}:
\begin{definition}[$\alpha$-representation on $\mathcal{M}_+$]
For $\alpha\in[-1,1]$, the $\alpha$-\emph{representation} of a positive measure $m\in\mathcal{M}_+$ is defined by the embedding $\pi^{(\alpha)}:\mathcal{M}_+\hookrightarrow\mathbb{R}_+^n$:
\begin{equation}
m\mapsto x=
\begin{cases}
    m^\frac{1-\alpha}{2}, &\alpha\ne 1\\
    \log m, &\alpha=1
\end{cases} .
\end{equation}
\end{definition}

\begin{definition}[$\alpha$-connection on $\mathcal{M}_+$]
For $\alpha\in[-1,1]$, the $\alpha$-\emph{connection} on $\mathcal{M}_+$ is given by
\begin{equation}
    \tilde\nabla^{(\alpha)}_A B :=\left(a_j\frac{\partial b_i}{\partial m_j} - \frac{1+\alpha}{2}\frac{a_i b_i}{m_i}\right)\partial m^i.
\end{equation}
\end{definition}
Let $\tilde\nabla^{(e)}:=\tilde\nabla^{(1)},\tilde\nabla^{(m)}:=\tilde\nabla^{(-1)}$, we have
\begin{equation}
    \tilde\nabla^{(\alpha)}=\frac{1-\alpha}{2} \tilde\nabla^{(m)}+\frac{1+\alpha}{2} \tilde\nabla^{(e)}.
\end{equation}

 Simplifying Eq.\ref{eqn:geodesic_eqn}, we can obtain the following $\alpha$-geodesic equation on $\mathcal{M}_+$:
\begin{equation}
    \ddot{\gamma}-\frac{1+\alpha}{2}\frac{\dot{\gamma}^{2}}{\gamma}=0,
\end{equation}
which gives the following closed-form solution:
\begin{equation}
    m_t=\gamma_{m_0,m_1}^{(\alpha)}(t)=\left((1-t)m_0^{1/p}+tm_1^{1/p}\right)^{p} \label{eqn:unconstraint_geo}
\end{equation}
for the boundary conditions $\gamma(0)=m_0,\gamma(1)=m_1$. With the initial conditions, the $\alpha$-geodesic equation in Eq.\ref{eqn:unconstraint_geo} gives the closed-form exponential map of:
\begin{equation}
    \exp_m^{(\alpha)}a=\left(m^{1/p}+\frac{m^{1/p}}{pm}a\right)^{p}.
\end{equation}

$\alpha$-representations are closely related to $\alpha$-\emph{connections} that define different geometry on the statistical manifold. Let $p=2/(1-\alpha)$, the $\alpha$-\emph{geometry} for positive measures can be understood as operating on such $\alpha$-representations with linear interpolation:
\begin{equation}
    m_t^{1/p}:=\gamma^{(\alpha)}_{m_0,m_1}(t)^{1/p}=(1-t)m_0^{1/p}+tm_1^{1/p} .
\end{equation}
One can see the $\alpha$-geometry for positive measures is flat (actually \emph{dually flat}, see Chapter 4.2 in \citet{amari2016information}). Nonetheless, working with $\mathcal{P}_+\subseteq \mathcal{M}_+$ requires additional projection, which we will discuss below.

\subsection{$\alpha$-Connections and $\alpha$-Geodesics on Statistical Manifold}\label{supp:alpha_conn}
We now proceed to the $\alpha$-geometry on the statistical manifold $\mathcal{P}_+$.

\begin{definition}[$\alpha$-connection on $\mathcal{P}_+$]
For $\alpha\in[-1,1]$, the $\alpha$-\emph{connection} on $\mathcal{P}_+$ is given by
\begin{equation}
    \nabla^{(\alpha)}_A B :=\left(a_j\frac{\partial b_i}{\partial \mu_j} - \frac{1+\alpha}{2}\left(\frac{a_i b_i}{\mu_i}-\mu_i\langle a,b\rangle_\mu \right)\right)\partial\mu^i.
\end{equation}
\end{definition}
$\alpha$-geodesic is defined as the geodesic induced by $\nabla^{(\alpha)}$. Simplifying Eq.\ref{eqn:geodesic_eqn}, we can obtain the following $\alpha$-geodesic equation on the statistical manifold:
\begin{equation}
    \ddot{\gamma}-\frac{1+\alpha}{2}\left(\frac{\dot{\gamma}^{2}}{\gamma}-\gamma \sum_{j} \frac{\dot{\gamma}_{j}^{2}}{\gamma_{j}}\right)=0 . \label{eqn:alpha_geodesic_eqn}
\end{equation}

It is reasonable to make a solution ansatz by normalization of the unconstrained $\alpha$-geodesics for positive measures in Eq.\ref{eqn:unconstraint_geo} to obtain the $\alpha$-geodesics on the statistical manifold. However, as noted in previous mathematical work \citep{morozova1991markov,bauer2024p}, both normalized curves of the exponential map and interpolation have to be reparameterized.
For solving the reparameterizations, \citet{bauer2024p} established the following theorems that allow for explicit solutions using numerical ODE solvers. As we have already defined the reversible mapping of $\alpha$-representation in Eq.\ref{eqn:mapping} and its corresponding vector field mapping in Eq.\ref{eqn:vf_mapping}, we will use $\tau_{x,u}:=\tau_{\mu,a},\tau_{x,y}:=\tau_{\mu,\nu}$ in the following context.

\begin{theorem}[$\tau_{x,u}$]\label{thm:reparam_exp}
Let $x\in S_p$ and $u\in T_x S_p$, and let $\tau_{x,u}$ satisfy the ODE with initial values:
\begin{align}
\ddot{\tau}(t) & =2 \frac{\sum_i(x_i+\tau(t) u_i)^{p-1} u_i}{\|x+\tau(t) u\|_p^{p}} \dot{\tau}(t)^{2}, \\
\tau(0) & =0, \\
\dot{\tau}(0) & =1,
\end{align}
then $\tilde\gamma:[0,1]\to S_p$ defined by
\begin{equation}
    \tilde\gamma(t)=\frac{x+\tau(t) u}{\|x+\tau(t) u\|_p}
\end{equation}
is the $\alpha$-geodesic on $S_p$, with initial conditions $\tilde\gamma(0)=x,\dot{\tilde\gamma}(0)=u$.
\end{theorem}

\begin{theorem}[$\tau_{x,y}$]\label{thm:reparam_int}
Let $x,y\in S_p$ and $\tau_{x,y}$ satisfy the ODE with boundary values:
\begin{align}
\ddot{\tau}(t) & =2 \frac{\sum_i(x_i+\tau(t) (y_i-x_i))^{p-1} (y_i-x_i)}{\|x+\tau(t) (y-x)\|_p^{p}} \dot{\tau}(t)^{2}, \\
\tau(0) & =0, \\
\tau(1) & =1,
\end{align}
then $\tilde\gamma:[0,1]\to S_p$ defined by
\begin{equation}
    \tilde\gamma(t)=\frac{x+\tau(t) (y-x)}{\|x+\tau(t) (y-x)\|_p}
\end{equation}
is the $\alpha$-geodesic on $S_p$, with boundary conditions $\tilde\gamma(0)=x,\tilde\gamma(1)=y$.
\end{theorem}

The $\alpha$-geodesic on $\mathcal{P}_+$ can be easily obtained by applying the inverse mapping of Eq.\ref{eqn:mapping} on the solved $\tilde\gamma(t)$ on $S_p$ as $\gamma=(\pi^{(\alpha)})^{-1}(\tilde\gamma)$, as demonstrated in Fig.\ref{fig:geodesic}.
The closed-form solutions of $\tau_{x,u},\tau_{x,y}$ for general $\alpha\in(-1,1)$ are unknown. Nonetheless, it is possible to obtain the solutions with numerical ODE solvers on the fly at the cost of increased training and inference time. See Appendix~\ref{supp:solve_reparam} for the detailed implementation.

\subsection{Closed-Form $\alpha$-Geodesics} \label{supp:solve_geodesic}
As discussed in Section~\ref{sec:classification}, closed-form solutions are known for the special cases when $\alpha=0,\pm 1$, making practical training feasible in these cases. We will elaborate further on how to deduce these solutions.

\paragraph{$m$-Geodesic.} For $\alpha=-1,p=1$, both ODEs can be simplified to $\ddot\tau=0$, giving the linear solution of $\tau(t)=t$ in both reparameterizations. This leads to the linear interpolation between categorical distributions. One can also directly solve the $\alpha$-geodesic equation in Eq.\ref{eqn:alpha_geodesic_eqn} with either the initial condition $\gamma(0)=\mu,\dot\gamma(0)=a$ or the boundary condition $\gamma(0)=\mu,\gamma(1)=\nu$, which gives exactly the same results.

\paragraph{Levi-Civita Geodesic.} For $\alpha=0,p=2$, the $\alpha$-connection is the Levi-Civita connection, which is compatible with the Fisher information metric. The mapped representations $x=\sqrt\mu$ lie on the unit sphere, and with such symmetry, we can assume one point always lies on the vertex of the simplex. The vector field can then be represented by a single scalar in the 2D intersection of the sphere, perpendicular to the base vector. For $\tau_{x,u}$, the ODE can be reduced to
\begin{equation}
    \ddot{\tau}(t)=\frac{2\tau u}{1+\tau^2 u^2}\dot{\tau}(t)^{2},
\end{equation}
which gives the solution of $\tau(t)=\tan(ut)/u$. One can also verify that the following solution
\begin{equation}
    \tau_{x,y}(t)=\frac{\tan t\theta}{\sin\theta+(1-\cos\theta)\tan t\theta},\quad\theta=\arccos(\langle x,y\rangle)
\end{equation}
satisfies the BVP and gives $\dot\tau_{x,y}(0)=\theta/\sin\theta$. With proper substitution, the general exponential and logarithm maps can be written as
\begin{align}
    \widetilde{\exp}^{(0)}_x u&=\frac{x+u\cdot\tan\|u\|_2/\|u\|_2}{1+\tan^2\|u\|_2}=x\cos\|u\|_2+\frac{u}{\|u\|_2}\sin\|u\|_2,\\
    \widetilde{\log}^{(0)}_x y&=\frac{\theta}{\sin\theta}P_x^{(0)}(y-x)=\frac{\arccos(\langle x,y\rangle)}{\sqrt{1-\langle x,y\rangle^2}}(y-x-\langle x,y-x\rangle x).
\end{align}
This is exactly the same as the standard exponential and logarithm maps on the unit sphere.

\paragraph{$e$-Geodesic.} The above two theorems cannot be applied to the limit case of $\alpha=1,p=\infty$. Nonetheless, $e$-geodesics have been well-explored in information geometry, and it is known that the normalization of the unconstrained $\alpha$-geodesics does give the correct $e$-geodesics in this case. One can verify that the following solution
\begin{equation}
    \gamma(t)=\softmax(\log \mu+at/\mu)
\end{equation}
does satisfy the $\alpha$-geodesic equation in Eq.\ref{eqn:alpha_geodesic_eqn}. See Proposition 2.5 in \citet{ay2017information} for a detailed deduction. This leads to closed-form exponential and logarithm maps as
\begin{align}
    \exp^{(e)}_\mu a&=\softmax(\log \mu+a/\mu),\\
    \log^{(e)}_\mu \nu&=\mu(\log\nu-\log\mu+D_\text{KL}(\mu\|\nu)).
\end{align}
After the $\alpha$-representation transform $\mu\mapsto \log \mu$ and the corresponding vector field transform $a\mapsto a/\mu$, these maps are exactly the same as those introduced in Table~\ref{tab:class}.

%% file: suppl/C_proof.tex
\section{Technical Details}
In this section, we provide additional mathematical details regarding the propositions and theorems in the main text.

\subsection{Continuity of $\alpha$-Geometry}\label{supp:alpha_limit}
In Definition~\ref{def:alpha_rep}, we define the $\alpha$-representation following previous work \citep{kimura2024density,amari2000methods,ay2017information}. However, one may notice a discontinuity at $\alpha=1,p=\infty$. We now further elaborate on the continuity of $\alpha$-flow at $\alpha=1$.

\paragraph{$\alpha$-Representation}
The discontinuity of $\alpha$-Representation in the main context can be avoided by using a slightly modified definition of $\alpha$-representation as
\begin{equation}
    \hat{\pi}^{(\alpha)}:\mu\mapsto \frac{2}{1-\alpha}\left(\mu^{(1-\alpha)/2}-1\right)=p(\mu^{1/p}-1) . \label{eqn:alt_mapping}
\end{equation}
The above definition coincides with Tsallis's $q$-logarithm $\log_q \mu$ with $q=(1+\alpha)/2$ and only differs by a constant scaling and translation from our original definition. The limit $\lim_{p\to\infty}p(\mu^{1/p}-1) =\log \mu$ holds for all $\mu>0$, which reduces to our original definition. The modified $\alpha$-representations of a Bernoulli distribution are visualized in Figure~\ref{fig:representation}, in which a continuous family of functions (i.e., homotopy of geodesics) can be observed.
Similarly, for the vector field mapping, taking the derivative with respect to $t$ on both sides of Eq.\ref{eqn:alt_mapping}, we obtain a unified mapping:
\begin{equation}
    u=\frac{\mu^{1/p}}{\mu}a,
\end{equation}
which further reduces to $u=a/\mu$ for $p=\infty$.

\begin{figure}[ht]
    \centering
    \includegraphics[width=0.4\linewidth]{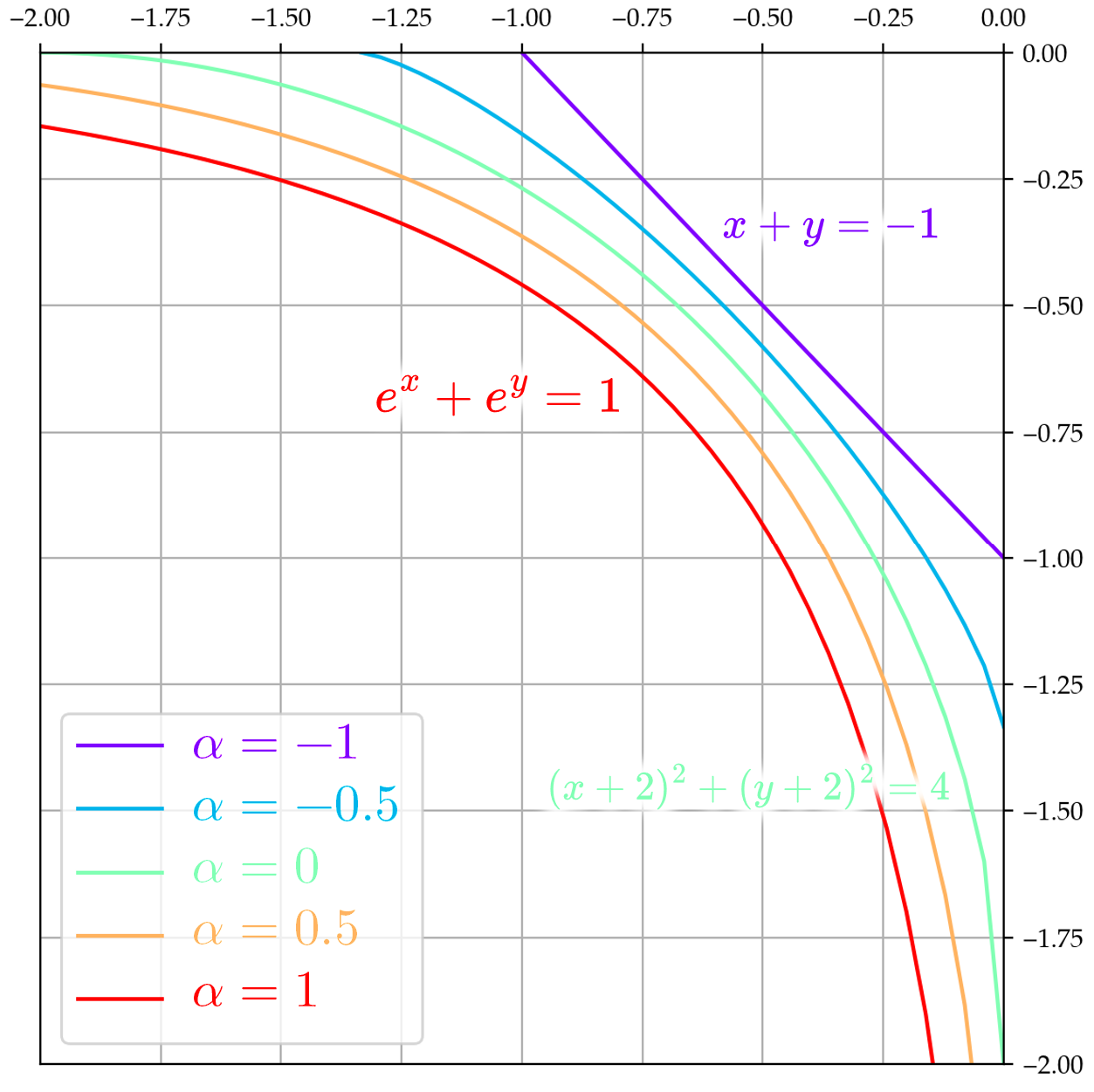}
    \vspace{-0.5em}
    \caption{Different modified $\alpha$-representations of a Bernoulli distribution.}
    \label{fig:representation}
\end{figure}

\paragraph{Exponential and Logarithm Maps}
We now discuss the limit case for the original exponential and logarithm maps in Eq.\ref{eqn:ori_exp} and \ref{eqn:ori_log}. Consider the numerator in Eq.\ref{eqn:ori_exp} under the limit case when $p\to\infty$, we have:
\begin{equation}
\begin{aligned}
    \left(\mu^{1/p}+\frac{\tau\mu^{1/p}}{p\mu}\right)^p
    &=\left(1+\frac{\hat{\pi}^{(\alpha)}(\mu)}{p}+\frac{\hat{\pi}^{(\alpha)}(\mu)+p}{p^2\mu}a\right)^p\\
    &=\left(1+\frac{\hat{\pi}^{(\alpha)}(\mu)+a/\mu}{p}+\frac{\hat{\pi}^{(\alpha)}(\mu)}{p^2\mu}a\right)^p\\
    &\to e^{\hat{\pi}^{(\alpha)}(\mu)+a/\mu}=\mu e^{a/\mu} .
\end{aligned}
\end{equation}
In the limit, we use $(1+a/p)^p\to e^a$ as $p\to\infty$. Therefore, we have
\begin{equation}
    \exp^{(e)}_\mu a=\frac{\mu e^{a/\mu}}{\sum_j\mu_j e^{a_j/\mu_j}}=\softmax(\log \mu+a/\mu) .
\end{equation}

For the logarithm map in Eq.\ref{eqn:ori_log}, let $w:=\mu/\nu$, we have
\begin{equation}
\begin{aligned}
    P_\mu\left(\mu\hat\pi^{(\alpha)}(w)\right)
    &=\mu\hat\pi^{(\alpha)}(w)-\mu\sum_j\mu_j\hat\pi^{(\alpha)}(w)_j\\
    &=p\mu\left(w^{1/p}-1\right)-\mu\sum_jp\mu_j\left(w_j^{1/p}-1\right)\\
    &=p\mu w^{1/p}-p\mu\sum_j \mu_j w_j^{1/p}-p\mu+p\mu\sum_j \mu_j\\
    &=p\left(\mu w^{1/p}-\mu\sum_j \mu_j w_j^{1/p}\right)=pP_\mu(\mu w^{1/p}) .
\end{aligned}
\end{equation} 
Therefore, the following identity holds:
\begin{equation}
    \log^{(\alpha)}_\mu\nu= \dot\tau_{\mu,\nu}(0)pP_\mu \left(\mu(\nu/\mu)^{1/p}\right)=\dot\tau_{\mu,\nu}(0) P_\mu\left(\mu \hat{\pi}^{(\alpha)}(\nu/\mu)\right). \label{eqn:log_identity}
\end{equation}
Taking the limit $\alpha\to 1$, we get
\begin{equation}
    \log^{(1)}_\mu\nu=P_\mu\left(\mu \log(\nu/\mu)\right)=\mu\log\frac{\nu}{\mu}-\mu\sum_j \mu_j\log\frac{\nu_j}{\mu_j},
\end{equation}
which recovers $\log^{(e)}_\mu(\nu)$. The interpolation along $e$-geodesic can therefore be calculated as:
\begin{equation}
    \gamma_{\mu,\nu}(t)=\exp^{(e)}_\mu \left(t \log^{(e)}_\mu\nu\right)=\softmax((1-t)\log\mu+t\log\nu).
\end{equation}

\paragraph{$\alpha$-Energy}
The optimality results we have established in Theorem~\ref{thm:local_optim} and \ref{thm:global_optim} require the definition of $\alpha$-energy in Eq.\ref{eqn:alpha_energy}, which is ill-defined for $p=\infty$ so far. To circumvent such an issue, we follow the standard technique to consider the rescaled energy functional defined as
\begin{equation}
    \tilde{E}^{(\alpha)}(\mu):=\left(\int_0^1 F^p(\gamma,\dot{\gamma}) \diff t\right)^{1/p}.
\end{equation}
The two functionals are related by $E^{(\alpha)}(\mu)=\frac{1}{p}\tilde{E}^{(\alpha)}(\mu)^p$, which is a strictly monotonic transformation that does not affect the optimality of the original functional. Therefore, taking the limit of the rescaled energy functional, we obtain:
\begin{equation}
    \tilde{E}^{(\infty)}(\mu)=\esssup_{t\in[0,1]}\|F(\gamma,\dot\gamma)\|_\infty=\esssup_{t\in[0,1]}\max_{1\le i\le n}\left|\frac{\dot\gamma_i}{\gamma_i}\right|,
\end{equation}
where $\esssup$ denotes the essential supremum. In this way, the optimality of $\alpha=1$ can be understood as finding the optimal smooth curves with the minimal value of the max-relative-velocity functional, which measures the worst-case instantaneous rate of change relative to the probability coordinate. 

So far, we have concluded the continuity of $\alpha$-geometry at $\alpha=1$ as a consistent family of generative models. As scaling and translation do not impact our theoretical claims, we adopt the simpler formulation of $\alpha$-geometry in the main text. 

\subsection{Proof for Proposition \ref{prop:div_taylor}} \label{supp:proof_div_taylor}
Before proving Proposition \ref{prop:div_taylor}, we need to first introduce the definition of $\alpha$-divergence on $\mathcal{M}_+$ for derivative calculation.

\begin{definition}[$\alpha$-divergence on $\mathcal{M}_+$] \label{def:alpha_div_m}
The following function $D^{(\alpha)}:\mathcal{M}_+\times \mathcal{M}_+\to\mathbb{R}_+$ is called the $\alpha$-\emph{divergence}:
\begin{equation}
    D^{(\alpha)}(m\|n):=\frac{2}{1-\alpha}\sum_{i}n_i+\frac{2}{1+\alpha}\sum_i m_i-\frac{4}{1-\alpha^2}\sum_{i}n_i^{\frac{1+\alpha}{2}}m_i^{\frac{1-\alpha}{2}} .
\end{equation}
Taking the limit, we obtain
\begin{equation}
    D^{(-1)}(m\|n)=D_\text{KL}(m\|n)=\sum_{i}n_i-\sum_i m_i+\sum_{i}m_i\log\frac{m_i}{n_i},\quad
    D^{(1)}(m\|n)=D_\text{KL}(n\|m) .
\end{equation}
\end{definition}

The restriction of the above $\alpha$-divergence to $\mathcal{P}_+\subseteq\mathcal{M}_+$ recovers Definition~\ref{def:alpha_div}. However, when calculating the partial derivatives, the results will be different. We shall use the above general definition.

\begin{proof}[Proof for Proposition \ref{prop:div_taylor}]
Taking the partial derivatives with respect to the second argument $n$ on the generalized KL divergence in Definition~$\ref{def:alpha_div_m}$, we obtain:
\begin{equation}
    \frac{\partial}{\partial n}D_\text{KL}(m\|n)=1-\frac{m}{n}.
\end{equation}
Restriction on $\mathcal{P}_+$ and $\mu=\nu$ gives
\begin{equation}
     \left.\frac{\partial}{\partial \nu}D_\text{KL}(\mu\|\nu)\right|_{\nu=\mu}=1-\frac{\mu}{\nu}=0 .
\end{equation}
This indicates all first-order partial derivatives are zero for local $D_\text{KL}(\mu\|\nu)$. This is also a direct result of the fact that $D_\text{KL}(\mu\|\nu)$ achieves the global minimum value of 0 when $\mu=\nu$.
Further taking the second-order derivatives with respect to the second argument $n$, we get
\begin{equation}
    \frac{\partial^2}{\partial n^2}D_\text{KL}(m\|n)=\frac{m}{n^2}.
\end{equation}
All the mixed second-order partial derivatives vanish. The restriction now gives:
\begin{equation}
     \left.\frac{\partial^2}{\partial \nu^2}D_\text{KL}(\mu\|\nu)\right|_{\nu=\mu}=\frac{\mu}{\nu^2}=\frac{1}{\mu} .
\end{equation}
Therefore, the second-order Tayler expansion of $D_\text{KL}(\mu\|\nu)$ at $\nu=\mu$ gives
\begin{equation}
    D_\text{KL}(\mu\|\mu+\mathrm{d}\mu)\approx \sum_i\frac{\mathrm{d}\mu_i^2}{\mu_i}=\|\mathrm{d}\mu\|_g^2,
\end{equation}
with a higher-order error of $o(\|\mathrm{d}\mu\|_g^2)$.
\end{proof}

\begin{remark}
It is also possible to directly calculate the second-order derivatives of the standard definition of the KL divergence for two categorical distributions with respect to $\mu$, as the additional linear terms in Definition~\ref{def:alpha_div_m} vanish in second-order derivatives.
\end{remark}

\begin{remark}
With similar calculation, one can demonstrate that a similar result holds for all $\alpha$-divergence:
\begin{equation}
    D^{(\alpha)}(\mu\|\mu+\mathrm{d}\mu)=\frac{1}{2}\|\mathrm{d}\mu\|_g^2 +o\left(\|\mathrm{d}\mu\|_g^2\right),\quad \mu\in\mathcal{P}_+ .
\end{equation}
This indicates that all $\alpha$-geometries are built upon the same Fisher information metric. Only third-order partial derivatives of $D^{(\alpha)}$ (which coincide with the Christoffel symbols) can be used to distinguish the different geometries.
\end{remark}

\subsection{Proof for Proposition \ref{prop:gradient}} \label{supp:proof_grad}
\begin{proof}
Taking the partial derivatives with respect to the first argument $m$ on the generalized $(-\alpha)$-divergence in Definition \ref{def:alpha_div_m}, we obtain:
\begin{equation}
    \frac{\partial}{\partial m}D^{(-\alpha)}(m\|n)=\frac{2}{1-\alpha}\left(1-\left(\frac{m}{n}\right)^{(1-\alpha)/2}\right)=p\left(1-\left(\frac{m}{n}\right)^{1/p}\right).
\end{equation}
The restriction on $\mathcal{P}_+$ gives
\begin{equation}
    \frac{\partial}{\partial \mu}D^{(-\alpha)}(\mu\|\nu)=p\left(1-\left(\frac{\mu}{\nu}\right)^{1/p}\right)=-\hat\pi^{(\alpha)}(\mu/\nu) \label{eqn:partial_div}
\end{equation}
where $\hat\pi^{(\alpha)}$ is the modified $\alpha$-representation in Eq.\ref{eqn:alt_mapping}.
From \citet{ay2017information}, Chapter 2.3, the gradient field on the statistical manifold can be obtained as
\begin{equation}
    \grad_\mu A=\mu\frac{\partial A}{\partial\mu}.
\end{equation}
Therefore, we have
\begin{equation}
    P_\mu\left(-\grad_\mu D^{(-\alpha)}(\cdot\|\nu)\right)=P_\mu\left(\mu \hat{\pi}^{(\alpha)}(\nu/\mu)\right).
\end{equation}
Comparison to Eq.\ref{eqn:log_identity} immediately concludes the proof.
\end{proof}

\begin{remark}
This relation between the partial derivatives of the $\alpha$-divergence and the $\alpha$-representation in Eq.\ref{eqn:partial_div} is not coincidental. In fact, the standard way of defining $\alpha$-divergence for positive measures follows the gradient-based approach such that
\begin{equation}
    \log_m n=-\grad_m D_n(m).
\end{equation}
See \citet{ay2017information}, Chapter 2.7.1 for details. The restriction on $\mathcal{P}_+$ only differs by the projection onto the tangent space $T_\mu \mathcal{P}$ and the scaling factor of $\dot\tau_{\mu,\nu}(0)$.
\end{remark}

%% file: suppl/D_model.tex
\section{Model Parameterization}
This section provides additional details regarding the model parameterization and sampling.

\subsection{Solving for Reparameterization}\label{supp:solve_reparam}
For the $\alpha$ values other than $\pm1,0$, we need to explicitly solve the reparameterizations $\tau_{\mu,\nu},\tau_{\mu,a}$ on the fly as described in Appendix~\ref{supp:alpha_conn}. We use the Euler method with 100 steps to solve the initial value problem (IVP) for $\tau_{x,u}$ and use the shooting method to solve the boundary value problem (BVP) for $\tau_{x,y}$. Furthermore, noting that $\tau(t)$ is monotone in $t$, we binary-search for the best $\dot\tau(0)$ values from the range of $[1,2]$ with up to 10 iterations and a tolerance of $10^{-3}$ for early stopping. In this way, we are able to batch the calculation.

It is clear to see that the additional computational overhead caused by solving the reparameterization does not rely on the model. Therefore, for larger models, such an overhead is less significant. Indeed, in Table~\ref{tab:time}, we provide a quantitative benchmark on the computational overheads of $\alpha=\pm0.5$ on different datasets, where it takes about 20-50\% more time, which is still practically feasible. The overheads in the GPU memory are also negligible (less than 1GB or even the same) compared to closed-form geodesics and DS-DFM models.
We also note that most of the overhead comes from solving the BVP, where multiple full rounds of the Euler solver may be required. However, with a more advanced BVP solver (e.g., \citet{bvpsolver} that utilizes the Jacobian information), it is possible to further accelerate the process, reducing the computational overhead further.

\input{tabs/time}
 
\subsection{Vector Field Parameterization}
As discussed in Section~\ref{sec:alpha_geometry}, our proposed $\alpha$-flow learns the mapped vector field in Eq.\ref{eqn:vf_mapping} on the mapped $\alpha$-geodesic for generative modeling. Following \citet{chen2023riemannian}, we project the predicted vector field to the tangent space $T_xS_p$ as
\begin{equation}
    v_\theta(x_t,t):=P^{(\alpha)}_{\mu_t}\big(\tilde{v}_\theta(x_t,t)\big)
\end{equation}
where $\tilde{v}_\theta$ is the unconstrained predictor with the projection operator $P_\mu^{(\alpha)}$ defined in Eq.\ref{eqn:proj_vf}. This ensures that the final vector field lies in the tangent space, and the exponential map is well-defined. Compared with Eq.\ref{eqn:mapped_log}, the unconstrained vector field effectively learns the linear difference $y-x$ just as the standard flow matching model on the Euclidean space.

The linear flow matching also follows a similar parameterization where we subtract the mean from the predicted vector field to ensure it lies on the tangent space. The standard Euclidean norm is applied to measure the difference in vector fields. 
On the other hand, for the log-linear flow matching, we follow \citet{mahabadi2023tess} to assign the target one-hot distribution with a logit of $n\delta-1$, where $n$ is the number of discrete classes. Unlike other CS-DFM models where the prior noise distribution $p_0$ is defined as the uniform distribution over the simplex, the prior distribution on the logits for the log-linear flow matching follows the standard normal distribution, as proposed in \citet{mahabadi2023tess}. With the Euclidean assumption, the vector field on the logits space has no constraint, and the Euclidean norm is applied to calculate the loss.

We also note that, although the $\alpha$-geodesics and the corresponding exponential and logarithm maps are well-defined on the positive categorical distributions on $\mathcal{P}_+$, we may encounter numerical issues near the boundary.
Therefore, when calculating the Riemannian norm, we clamp $\mu$ with a minimum value of $10^{-3}$ for the mixture-class model. For the exponential-class model, the minimum $\mu$ is set to $10^{-3}$ such that the logit mapping $\pi^{(1)}$ will be meaningful. The vector fields are also calculated in a numerically stable way to avoid undefined behavior.
For DS-DFM models, $t$ is clamped with a maximum value of 0.999 to avoid division by zero when calculating the loss.

\subsection{Discrete-State DFM Parameterization}
For a fair comparison between CS- and DS-DFM models, we use the same encoder architecture except for the additional mask token for the latter. In this way, the output of the unconstrained predictor for DS-DFM models is treated as logits. Following the standard masked diffusion formulation \citep{sahoo2024simple,shi2024simplified,gat2024discrete}, the DS-DFM loss can be written as (assuming linear interpolation paths):
\begin{equation}
    \mathcal{L}=-\mathbb{E}_{t,q(\mu)}\left[\frac{1}{1-t}\sum_{\delta_t=m}\log \softmax v_\theta(\delta_t,t)\right],
\end{equation}
where $\delta_t\sim \mu_t=(1-t)m+t\mu_1$ is the intermediate one-hot distribution and $m$ is the one-hot mask distribution. Note that the summation is only over the masked tokens. For masked DS-DFM, the prior noise is deterministically set to the mask token.

\subsection{Sampling from DFM Models}
The sampling of $\alpha$-flow follows exactly the same procedure as any other Riemannian flow matching models. We provide the sampling algorithm using the Euler method in Algorithm~\ref{alg:sample}. We do not use off-the-shelf adaptive ODE solvers to ensure that the trajectory does not leave the manifold and that the exponential maps in Eq.\ref{eqn:mapped_exp} are correctly calculated.

\begin{algorithm}[ht]
\caption{Sampling from $\alpha$-flow}\label{alg:sample}
\begin{algorithmic}[1]
\STATE Sample $\mu_0\sim p_0(\mu)$, apply $x_0=\pi^{(\alpha)}(\mu_0)$ in Eq.\ref{eqn:mapping}.
\FOR{$t\gets 0,1/N,\dots,1-1/N$}
    \STATE $x_{t+1/N}=\widetilde{\exp}_{x_0}^{(\alpha)}(v_\theta(x_t,t)/N)$ using Eq.\ref{eqn:mapped_exp}.
\ENDFOR
\STATE \textbf{Return:} Sample from $\mu_1=(\pi^{(\alpha)})^{-1}(x_1)$.
\end{algorithmic}
\end{algorithm}

We follow the original papers and their official code of MDLM \citep{sahoo2024simple} and DFM \citep{gat2024discrete} to implement their sampling process, which we will briefly outline here. For MDLM, a single-step update follows:
\begin{equation}
    \delta_{t+1/N}\sim \frac{(1-t-1/N)m+v_\theta(\delta_t,t)/N}{1-t}.
\end{equation}
For DFM, a single-step update follows:
\begin{equation}
    \delta_{t+1/N}\sim m\cdot e^{-1/N}+v_\theta(\delta_t,t)\cdot (1-e^{-1/N}).
\end{equation}
Both updates only apply to the masked tokens; any unmasked token will remain unchanged.

%% file: tabs/time.tex
\begin{table}[ht]
\centering
\caption{Computational overheads with respect to $\alpha=0$ during the training stage.}\label{tab:time}
\begin{tabular}{@{}lccc@{}}
\toprule
Overhead/+\% & BMNIST & UniRef50 & Text8 \\ \midrule
$\alpha=-0.5$ & 149 & 61 & 49 \\
$\alpha=0.5$ & 27 & 40 & 21 \\ \bottomrule
\end{tabular}
\end{table}

%% file: suppl/E_experiment.tex
\section{Experimental Setup}
This section provides additional details regarding the dataset and model architecture. Details of the training and evaluation stages are also provided.

\subsection{Swiss Roll}
The Swiss roll on the 2-simplex is adapted from \citet{cheng2024categorical}, in which the 2D Swiss roll is projected onto the 2-simplex to give a 3-category generative modeling task. We use a simple 2-layer multi-layer perceptron (MLP) as the predictor with additional time embeddings concatenated to the input for all CS-DFM models. We sampled 10k points as the training data and trained each model for 2000 epochs with full-batch gradient descent. 
We generated 10k points using 1000 Euler steps for each trained different instantiation of $\alpha$-flow with $\alpha=\pm1,\pm0.5, 0$.
To evaluate the generation quality, we used Gaussian kernel density estimation (KDE) with a bandwidth of 0.02 to estimate the density on the simplex (see Figure~\ref{fig:swissroll}). KL divergence was calculated between the estimated densities of the generation and the data.

\subsection{Binarized MNIST}
We followed \citet{cheng2024categorical} to use the preprocessed binarized MNIST dataset from \citet{salakhutdinov2008quantitative}, which has a split of 50k/10k/10k. We used the same CNN-based predictor from \citet{cheng2024categorical}, which is, in turn, adapted from \citet{song2020improved}. The resulting model has about 30M trainable parameters. Each model was trained on a single NVIDIA A100 GPU for 100k iterations with a batch size of 256 and an initial learning rate of $3\times10^{-4}$, which took about 10 hours.

The calculation of FID scores also followed \citet{cheng2024categorical}, in which 1000 images were sampled for each model using 300 Euler steps. The ground truth statistics were calculated on the whole training set using the pre-trained InceptionV3 \citep{szegedy2016rethinking} model.

\subsection{Text8}
The Text8 vocabulary consists of the 26 lowercase letters and the whitespace token. We followed \citet{campbell2024generative,shi2024simplified} to use a split of 90M/5M/5M characters and a fixed sequence length of 256 with the sliding window when sampling from the training set. Following \citet{cheng2024categorical}, the DiT-based predictor has about 92M trainable parameters. Each model was trained on 8 NVIDIA A100 GPUs for 500k iterations with a total batch size of 1M characters, totaling up to 500B characters. We used an initial learning rate of $3\times 10^4$. The training stage for each model took about 5 days.

Following \citet{campbell2024generative,cheng2024categorical}, we utilized the pre-trained GPT-J-6B model \citep{gpt-j} as a model-agnostic approach for the quantitative evaluation of the generation quality. A total of 1M characters were sampled for each model using 256 Euler steps and were then tokenized with the original tokenizer for GPT-J-6B. NLL and entropy were calculated on the token level. 
Note that since the GPT-J-6B model has not been trained on Text8, its NLLs do not necessarily reflect the fitness of the data distribution. Previous work \citep{campbell2024generative} indicates that these NLLs can be fooled into arbitrarily low values by repeating high-frequency but meaningless words. Still, such a comparison can provide insightful results.

\subsection{UniRef50}
We used the split in \citet{alamdari2023protein}, which has about 42M sequences in the training set. We followed \citet{wang2024diffusion,wang2024dplm} to use a maximum sequence length of 1024 (including the \texttt{<cls>} and \texttt{<eos>} tokens, so the pure protein sequence has a maximum length of 1022), and proteins longer than 1022 were chunked into shorter sequences of length 1022. As the protein sequence length may vary, we followed \citet{alamdari2023protein} to batch proteins of similar lengths together for efficiency. The ESM-2 tokenizer \citep{lin2023evolutionary} was used for all models, with a vocabulary size of 33 (including the mask token and other special tokens).
We used a 12-layer DiT model as the predictor with about 92M trainable parameters. Each model was trained for 200k iterations with a maximum batch size of 800k tokens (which may be smaller due to the batching strategy), totaling up to 160B tokens. We used an initial learning rate of $3\times 10^4$. The training stage for each model took about 40 hours on 8 NVIDIA A100 GPUs.

For each length $n\in[100,200,300,400,500,1000]$, we sampled 100 sequences from each model with a specified length of $n+2$ (to allow for the \texttt{<cls>} and \texttt{<eos>} tokens) and 256 Euler steps. The generated sequences were post-processed as follows:
\begin{enumerate}[topsep=0pt,itemsep=-1ex,partopsep=1ex,parsep=1ex]
    \item The first \texttt{<cls>} token and the last \texttt{<eos>}, if generated, were stripped.
    \item If any \texttt{<eos>} token appeared in the last 10 positions, it was stripped with any token after it. This is because, during training, we padded protein sequences with different lengths for batched training. Therefore, it is possible that the model will generate the \texttt{<eos>} early.
    \item Any other special tokens were substituted with \texttt{X}, the unknown amino acid.
\end{enumerate}
Following \citet{wang2024diffusion,wang2024dplm}, ESMFold \citep{lin2023evolutionary} was used to fold the proteins and calculate the pLDDT scores to measure the foldability of the generations. Following \citet{frey2023protein}, the Fréchet ESM distance (FED) was calculated using the 650M version of ESM-2 \citep{lin2023evolutionary}. We extracted the last hidden activations of the \texttt{<cls>} token as the sequence representation for calculating the statistics. The ground truth statistics were calculated on the sequences whose length is exactly $n$ in the training set. This gives 151708/90044/59531/37513/21528/3004 sequences, respectively, for each length.

%% file: suppl/F_results.tex
\section{Additional Visualizations}\label{supp:visualization}
We provide additional visualizations in this section.

\begin{figure}[htb]
    \centering
    \includegraphics[width=.5\linewidth]{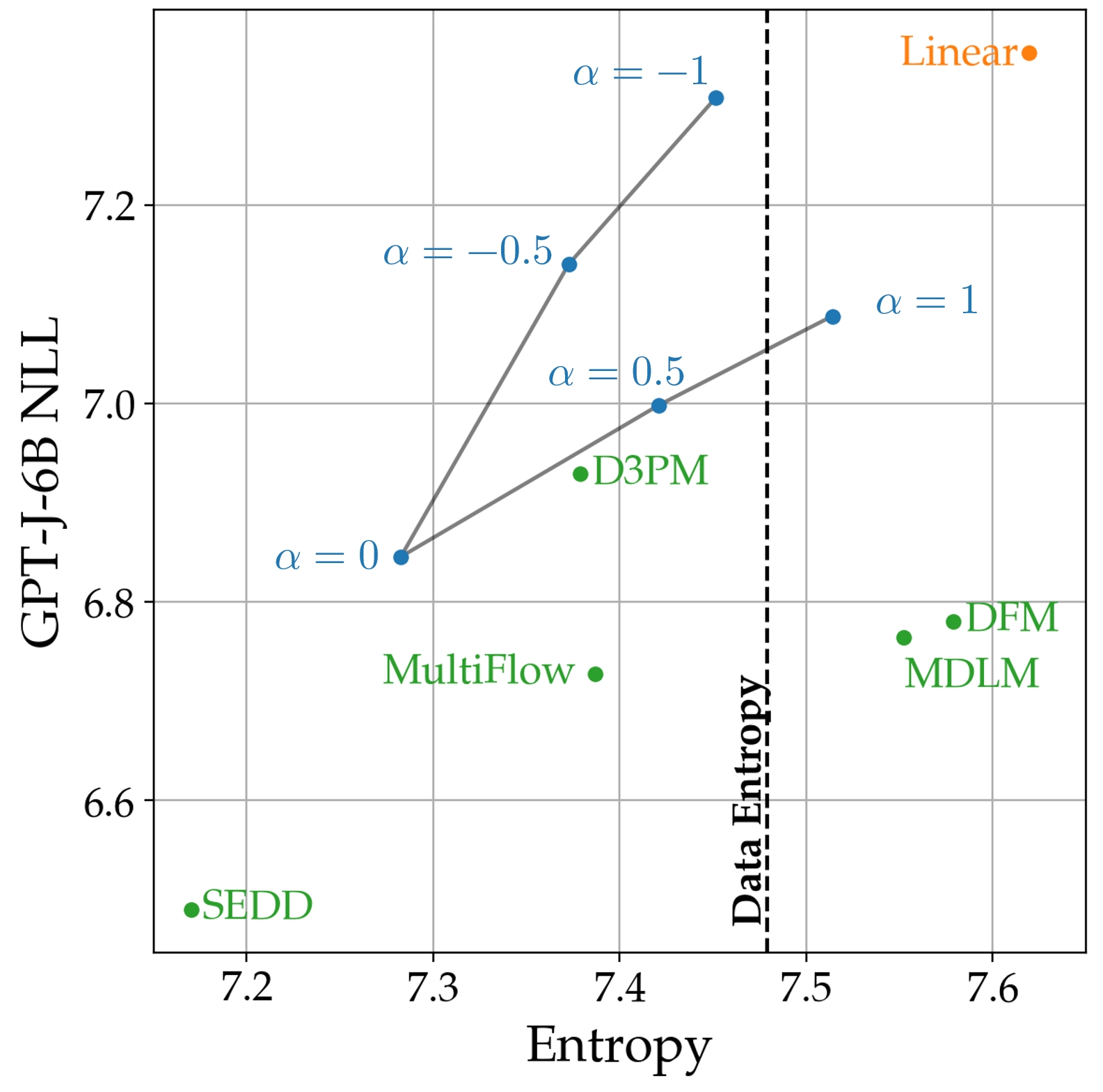}
    \vspace{-0.5em}
    \caption{Generation NLLs (evaluated by GPT-J-6B) vs entropy for different discrete generative models. Different $\alpha$-flow variants are colored in blue; DS-DFM models are colored in green.}
    \label{fig:text8}
\end{figure}

In Figure~\ref{fig:text8}, we provide additional visualization of the GPT-J-6B NLL and the entropy of the generated samples for each model. The data entropy is plotted as a vertical line, and an entropy closer to the data entropy is preferred. The five variants of $\alpha$-flow are connected, in which we can observe the impact of $\alpha$ on the trade-off between NLL and entropy. Specifically, $\alpha=0$ achieved the best NLL, whereas $\alpha=\pm1$ achieved the best entropy.

\begin{figure}[htb]
    \centering
    \includegraphics[width=\linewidth]{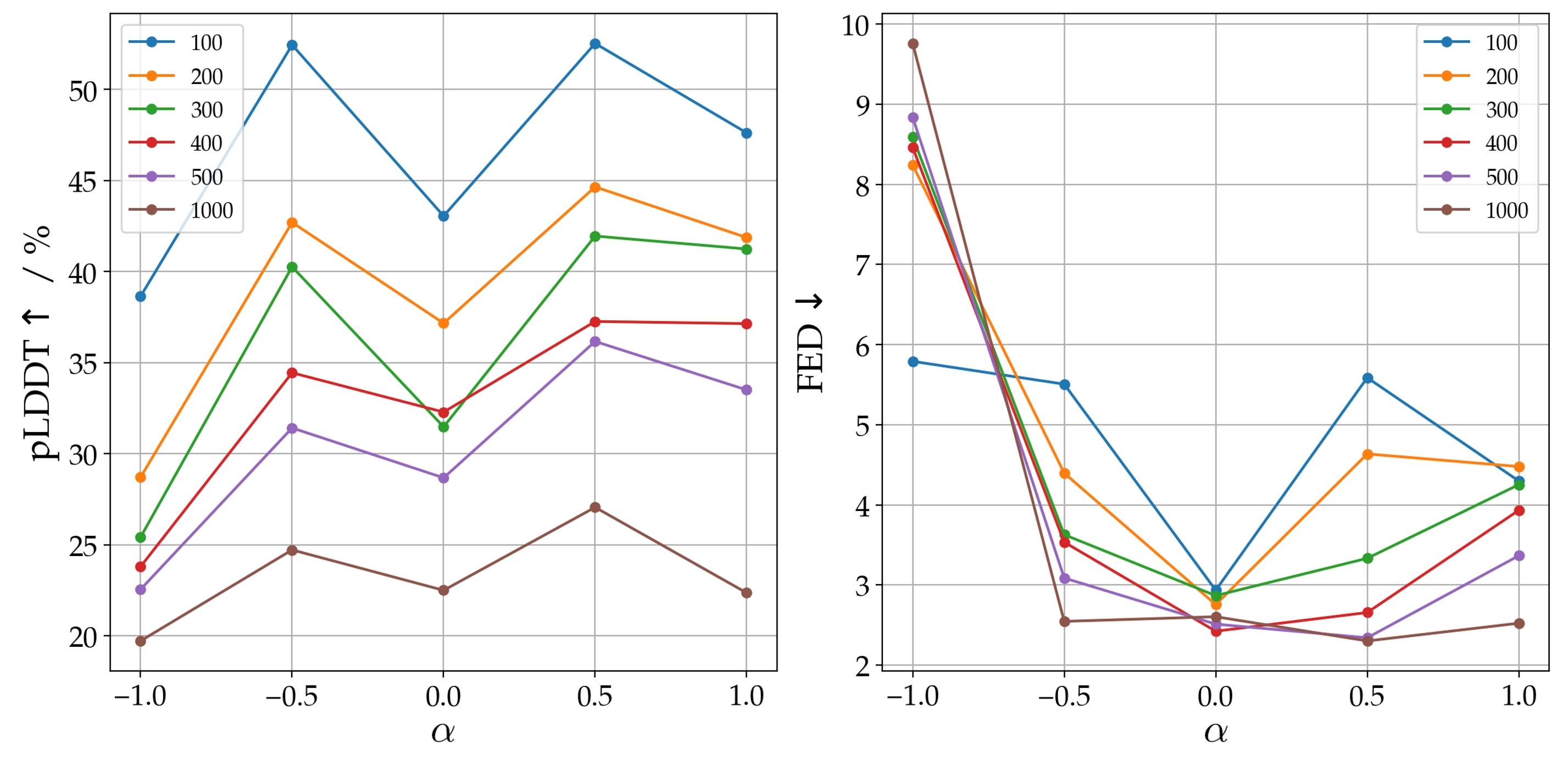}
    \vspace{-1em}
    \caption{pLDDT and FED scores with respect to different $\alpha$ values.}
    \label{fig:uniref_metric}
\end{figure}

In Figure~\ref{fig:uniref_metric}, we provide an alternative visualization of the impact of $\alpha$ on the metrics of pLDDT and FED. It can be demonstrated more clearly that $\alpha=0$ consistently achieves good FED scores on both shorter and longer protein sequences. On the other hand, $\alpha\pm0.5$ achieve better pLDDT scores, especially on shorter proteins.
